\documentclass[10pt]{article} %
\input{macros.sty}

\usepackage[margin=1in]{geometry}
\usepackage{parskip}

\usepackage[round, compress]{natbib}

\usepackage{booktabs}
\usepackage{multirow}
\usepackage{subcaption}

\usepackage{url}

\title{Improved Online Conformal Prediction via Strongly Adaptive Online Learning}

\author{
  Aadyot Bhatnagar\thanks{Salesforce AI Research. Email: \texttt{\{abhatnagar, huan.wang, cxiong, yu.bai\}@salesforce.com}.}
  \and
  Huan Wang\footnotemark[1]
  \and
  Caiming Xiong\footnotemark[1]
  \and
  Yu Bai\footnotemark[1]
}

\def\shownotes{0}  %
\ifnum\shownotes=1
\newcommand{\authnote}[2]{{\scriptsize $\ll$\textsf{#1 notes: #2}$\gg$}}
\else
\newcommand{\authnote}[2]{}
\fi
\newcommand{\yub}[1]{{\color{red}\authnote{Yu}{#1}}}

\usepackage{hyperref}
\hypersetup{
    colorlinks,
    linkcolor={blue!70!black},
    citecolor={blue!70!black},
}
\colorlet{linkequation}{blue}

\begin{document}

\maketitle

\begin{abstract}

We study the problem of uncertainty quantification via prediction sets, in an online setting where the data distribution may vary arbitrarily over time. Recent work develops {\em online conformal prediction} techniques that leverage regret minimization algorithms from the online learning literature to learn prediction sets with approximately valid coverage and small regret. However, standard regret minimization could be insufficient for handling changing environments, where performance guarantees may be desired not only over the full time horizon but also in all (sub-)intervals of time.
We develop new online conformal prediction methods that minimize the {\em strongly adaptive regret}, which measures the worst-case regret over all intervals of a fixed length. We prove that our methods achieve near-optimal strongly adaptive regret for all interval lengths simultaneously, and approximately valid coverage. 
Experiments show that our methods consistently obtain better coverage and smaller prediction sets than existing methods on real-world tasks, such as time series forecasting and image classification under distribution shift.

\end{abstract}
\section{Introduction}
Modern machine learning models make highly accurate predictions in many settings. In high stakes decision-making tasks, it is just as important to estimate the model's uncertainty by quantifying how much the true label may deviate from the model's prediction. A common approach for uncertainty quantification is to learn \emph{prediction sets} that associate each input with a set of candidate labels, such as prediction intervals for regression, and label sets for classification. The most important requirement for learned prediction sets is to achieve valid \emph{coverage}, i.e. they should cover the true label with at least $1-\alpha$ (such as $90\%$) probability. 

Conformal prediction~\citep{vovk2005alrw} is a powerful framework for augmenting any \emph{base predictor} (such as a pretrained model) into prediction sets with valid coverage guarantees~\citep{angelopoulos-gentle}. These guarantees require almost no assumptions on the data distribution, except \emph{exchangeability} (i.i.d.\ data is a sufficient condition). However, exchangeability fails to hold in many real-world settings such as time series data~\citep{chernozhukov18a} or data corruption~\citep{hendrycks2018using}, where the data may exhibit \emph{distribution shift}. Various approaches have been proposed to handle such distribution shift, such as reweighting~\citep{tibshirani2019cov_shift,barber2022nexcp} or distributionally robust optimization~\citep{cauchois2022knowing}.

A recent line of work develops \emph{online conformal prediction} methods for the setting where the data arrives in a sequential order~\citep{gibbs2021aci,gibbs2022faci,zaffran2022ts_aci,feldman2022risk}. At each step, their methods output a prediction set parameterized by a single \emph{radius} parameter that controls the size of the set. After receiving the true label, they adjust this parameter adaptively via \emph{regret minimization} techniques---such as Online Gradient Descent (OGD)~\citep{zinkevich2003olo}---on a certain quantile loss over the radius. These methods are shown to achieve empirical coverage frequency close to $1-\alpha$, regardless of the data distribution~\citep{gibbs2021aci}. In addition to coverage, importantly, these methods achieve sublinear regret with respect to the quantile loss~\citep{gibbs2022faci}. Such regret guarantees ensure that the size of the prediction set is reasonable, and rule out ``trivial'' algorithms that achieve valid coverage by alternating between predicting the empty set and full set (cf.\ Section~\ref{sec:related_conformal_ts} for a discussion).

While regret minimization techniques achieve coverage and regret guarantees, they may fall short in more dynamic environments where we desire a strong performance not just over the entire time horizon (as captured by the regret), but also within every \emph{sub-interval} of time. For example, if the data distribution shifts abruptly for a few times, we rather desire strong performance within each contiguous interval between two consecutive shifts, in addition to the entire horizon. \citet{gibbs2022faci} address this issue partially by proposing the Fully Adaptive Conformal Inference (FACI) algorithm, a meta-algorithm that aggregates multiple \emph{experts} (base learners) that are OGD instances with different learning rates. However, their algorithm may not be best suited for achieving such interval-based guarantees, as each expert still runs over the full time horizon and is not really localized. This is also reflected in the fact that FACI achieves a near-optimal $\tO(\sqrt{k})$ regret within intervals of a fixed length $k$, but is unable to achieve this over all lengths $k\in[T]$ simultaneously.

In this paper, we design improved online conformal prediction algorithms by leveraging \emph{strogly adaptive regret minimization}, a technique for attaining strong performance on all sub-intervals simultaneously in online learning~\citep{daniely2015saol, jun2017cbce}. Our proposed algorithm, Strongly Adaptive Online Conformal Prediction (\method{}), is a new meta-algorithm that manages multiple experts, with the key difference that each expert now only operates on its own \emph{active interval}. We summarize our contributions:
\begin{itemize}[leftmargin=1em,topsep=0pt,itemsep=0pt]
\item We propose \method{}, a new algorithm for online conformal prediction. \method{} is a meta-algorithm that maintains multiple experts each with its own active interval, building on strongly adaptive regret minimization techniques (Section~\ref{section:saocp}). We instantiate the experts as Scale-Free OGD (\methodBasic{}), an anytime variant of the OGD, which we also study as an independent algorithm.
\item We prove that \method{} achieves a near-optimal strongly adaptive regret of $\tO(\sqrt{k})$ regret over all intervals of length $k$ simultaneously, and that both \method{} and \methodBasic{} achieve approximately valid coverage (Section~\ref{sec:theory}).
\item We show experimentally that \method{} and \methodBasic{} attain better coverage in localized windows and smaller prediction sets than existing methods, on two real-world tasks: time series forecasting and image classification under distribution shift (Section~\ref{sec:experiments}).\footnote{The code for our experiments can be found at \url{https://github.com/salesforce/online_conformal}.}
\end{itemize}

\subsection{Related work}

\paragraph{Conformal prediction}
The original idea of conformal prediction (utilizing exchangeable data) is developed in the early work of~\citet{vovk1999original,vovk2005alrw,shafer2008tutorial}. Learning prediction sets via conformal prediction has since been adopted as a major approach for uncertainty quantification in regression~\citep{papadopoulos2008inductive, vovk2012conditional, lei2014distribution, vovk2018cross, romano2019cqr, gupta2019nested, barber2021jackknife, barber2022nexcp} and classification~\citep{lei2013classification, romano2020aps, cauchois2020robust, cauchois2022knowing, angelopoulos2021uncertainty}, with further applications in general risk control~\citep{bates2021distribution,angelopoulos2021learn,angelopoulos2022conformal}, biological imaging~\citep{angelopoulos2022image}, and protein design~\citep{fannjiang2022conformal}, to name a few. 

Recent work also proposes to optimize the prediction sets' {\em efficiency} (e.g.\ width or cardinality) in addition to coverage \citep{pearce2018high,Park2020PAC,yang2021efficient,stutz2022learning,angelopoulos2021learn,angelopoulos2021uncertainty,bai2022efficient}. The regret that we consider can be viewed as a (surrogate) measure for efficiency in the online setting.

\paragraph{Conformal prediction under distribution shift}
For the more challenging case where data may exhibit distribution shift (and thus are no longer exchangeable), several approaches are proposed to achieve approximately valid covearge, such as reweighting (using prior knowledge about the data's dependency structure)~\citep{tibshirani2019cov_shift, podkopaev2021label_shift, candes2021survival, barber2022nexcp}, distributionally robust optimization~\citep{cauchois2020robust}, or doubly robust calibration~\citep{yang2022robust}. 

Our work makes an addition to the online conformal prediction line of work~\citep{gibbs2021aci, gibbs2022faci, zaffran2022ts_aci, feldman2022risk}, which uses regret minimization techniques from the online learning literature~\citep{zinkevich2003olo, hazan2022oco} to adaptively adjust the size of the prediction set based on recent observations. Closely related to our work is the FACI algorithm of~\citet{gibbs2022faci}, which is a meta-algorithm that uses multiple experts for handling changing environments. Our meta-algorithm \method{} differs in style from theirs, in that our experts only operate on their own active intervals, and it achieves a better guarantee on the strongly adaptive regret.

A related line of work studies conformal prediction for time series data. %
\citet{chernozhukov18a, xu2021enbpi, sousa2022general} use randomization and ensembles to produce valid prediction sets for time series that are ergodic in a certain sense.
Some other works directly apply vanilla conformal prediction to time series either without theoretical guarantees or requiring weaker notions of exchangeability~\citep{dashevskiy2008network, wisniewski2020application, kath2021power,stankeviciute2021conformal_ts_exchangeable, sun2022copula}.

\paragraph{Strongly adaptive online learning}
Our algorithms adapt techniques from the online learning literature, notably strongly adaptive regret minimization~\citep{daniely2015saol, jun2017cbce, zhang2018dynamic} and scale-free algorithms for achieving other kinds of adaptive (e.g. anytime) regret guarantees~\citep{orabona2018ogd}.
\section{Preliminaries}
\label{sec:related_conformal_ts}

We consider standard learning problems in which we observe examples $(x,y)\in\cX\times\cY$ and wish to predict a label $y$ from input $x$. A \emph{prediction set} $C:\cX\to 2^\cY$ is a set-valued function that maps any input $x$ to a \emph{set} of predicted labels $C(x)\subset \cY$. Two prevalent examples are \emph{prediction intervals} for regression in which $\cY=\R$ and $C(x)$ is an interval, and \emph{label prediction sets} for ($m$-class) classification in which $\cY=[m]$ and $C(x)$ is a subset of $[m]$. Prediction sets are a popular approach to quantify the uncertainty associated with the point prediction $\hat{y} = f(x)$ of a black box model.

We study the problem of learning prediction sets in the \emph{online} setting, in which the data $(X_1,Y_1),\dots,(X_T,Y_T)$ arrive sequentially. At each time step $t\in[T]$, we output a prediction set $\hat{C}_t=\hat{C}_t(X_t)$ based on the current input $X_t$ and past observations $\set{(X_i,Y_i)}_{i\le t-1}$, \emph{before} observing the true label $Y_t$. The primary goal of the prediction set is to achieve \emph{valid coverage}: $\P[Y_t \in \hat{C}_t(X_t)] = 1-\alpha$, where $1-\alpha\in(0,1)$ is the target coverage level pre-determined by the user.
Standard choices for $\alpha$ include $\set{0.1,0.05}$, which correspond to $\set{90\%,95\%}$ target coverage respectively.

\paragraph{Online conformal prediction} We now review the idea of \emph{online conformal prediction}, initiated by~\citet{gibbs2021aci,gibbs2022faci}. This framework for learning prediction sets in the online setting achieves coverage guarantees even under distribution shift.

At each time $t \in [T]$, online conformal prediction assumes that we have a family of prediction sets $\cC_t=\{\hat{C}_t(x,s)\}_{x\in\cX,s\in\R}$ specified by a \emph{radius} parameter $s\in\R$, and we need to predict $\hat{s}_t\in\R$ and output prediction set $\hat{C}_t=\hat{C}_t(X_t, \hat{s}_t)\subset \cY$. The family $(\cC_t)_{t\in[T]}$ is typically defined through \emph{base predictors} $\hat{f}_t$ (for example, $\hat{f}_t\equiv f$ can be a fixed pretrained model). A standard example in regression is that we have a base predictor $\hat{f}_t:\cX\to\R$, and we can choose $\hat{C}_t(X_t, s) \defeq [\hat{f}_t(X_t) - s, \hat{f}_t(X_t) + s]$ to be a prediction interval around $\hat{f}_t(X_t)$, in which case the radius $s$ is the (half) width of the interval. In general, we allow any $\cC_t$
that are \emph{nested sets}~\citep{gupta2019nested} in the sense that $\hat{C}_t(x, s)\subseteq \hat{C}_t(x, s')$ for all $x\in\cX$ and $s\le s'$, so that a larger radius always yields a larger set.

Online conformal prediction adopts online learning techniques to learn $\hat{s}_t$ based on past observations. Defining the ``true radius'' $S_t\defeq \inf\{s\in\R: Y_t\in\hat{C}_t(X_t, s)\}$ (i.e. the smallest radius $s$ such that $\hat{C}_t$ covers $Y_t$), we consider the $(1-\alpha)$-\emph{quantile loss} (aka \emph{pinball loss}~\citep{koenker1978regression}) between $S_t$ and any predicted radius $\hat{s}$:
\begin{align}
    \ellt(\hat{s}) &= \ell_{1-\alpha}(S_t, \hat{s})
    \label{eq:quantile_loss}
    \defeq \max\{ (1-\alpha)(S_t - \hat{s}), \alpha(\hat{s} - S_t) \}.
\end{align}
Throughout the rest of this paper, we assume that all true radii are bounded: $S_t \in [0, D]$ almost surely for all $t\in[T]$. 

After observing $X_t$, predicting the radius $\hat{s}_t$, and observing the label $Y_t$ (and hence $S_t$), the gradient\footnote{More precisely, $\grad \ell^{(t)}(\hat{s}_t)$ is a subgradient.} $\grad\ellt(\hat{s}_t)$ can be evaluated and has the following simple form:
\begin{equation}
\label{eq:grad_loss}
\begin{aligned}
& \grad\ellt(\hat{s}_t) = \alpha - \indics{\hat{s_t} < S_t} = \alpha - \underbrace{\indics{Y_t \notin \hat{C}_t}}_{\defeq \err_t},
\end{aligned}
\end{equation}
where $\err_t$ is the indicator of miscoverage at time $t$ ($\err_t=1$ if $\hat{C}_t$ did not cover $Y_t$).
\citet{gibbs2021aci} perform an Online Gradient Descent (OGD) step to obtain $\hat{s}_{t+1}$:
\begin{align}
\label{eq:aci_ogd}
    \hat{s}_{t+1} = \hat{s}_t - \eta \grad \ell^{(t)}(\hat{s}_t) = \hat{s}_t + \eta(\err_t - \alpha),
\end{align}
where $\eta>0$ is a learning rate, and the algorithm is initialized at some $\hat{s}_1\in\R$. Update~\eqref{eq:aci_ogd} increases the predicted radius if $\hat{C}_t$ did not cover $Y_t$ ($\err_t=1$), and decreases the radius otherwise. This makes intuitive sense as an approach for adapting the radius to recent observations. 

\paragraph{Adaptive Conformal Inference (ACI)}
The ACI algorithm of~\citet{gibbs2021aci} uses update~\eqref{eq:aci_ogd} in conjunction with a specific choice of $\cC_t=\cC_t^{\aci}$, where
\begin{align}
    \hat{C}^{\aci}_t(X_t,\hat{s}_t)  =  \set{y: \wt{S}_t(X_t, y) \le Q_{\hat{s}_t}\paren{ \{\wt{S}_{\tau}\}_{\tau=1}^{t-1} } },  \label{eq:aci_ct}
\end{align}
where $\wt{S}_t:\cX\times\cY\to\R$ is any function (termed the \emph{score} function), $\wt{S}_\tau\defeq \wt{S}_\tau(X_\tau, y_\tau)$ denotes score of the $\tau$-th observation, and $Q_{s}(\cdot)$ denotes the $s\in(0,1)$-th empirical quantile of a set (cf.~\eqref{eq:def_Q}). In words, ACI's confidence set contains all $y$ whose score $\wt{S}_t(X_t, y)$ is below the $\hat{s}_t$-th quantile of the past scores, and they use~\eqref{eq:aci_ogd} to learn this quantile level.
The framework presented here generalizes the ACI algorithm where we allow any choice of $\cC_t$ that are nested sets, including but not limited to~\eqref{eq:aci_ct}. For convenience of discussions, unless explicitly specified, we will also refer to this general version of~\eqref{eq:aci_ogd} as ACI throughout this paper.

Empirically, we show in Section~\ref{sec:experiments} that FACI \citep{gibbs2022faci} (an extension of ACI) performs better when trained to predict $S_t$ directly under our more general parameterization, i.e. $\hat{C}(X_t, \hat{s}_t) = \{y : \wt{S}_{t}(X_t, y) \le \hat{s}_t\}$.

\begin{algorithm*}[t]
\small
\setcounter{AlgoLine}{0}
\DontPrintSemicolon
\SetNoFillComment
\SetKwInOut{Input}{Input}
\newcommand\mycommfont[1]{\footnotesize\ttfamily\textcolor{blue}{#1}}
\SetCommentSty{mycommfont}
\Input{Target coverage $1 - \alpha\in(0,1)$; maximum possible radius $D>0$}

\For{$t = 1, \ldots, T$}{
    \tcp{Obtain prediction interval by aggregating active experts}
    Initialize new expert $\mathcal{A}_t = \texttt{SF-OGD}\big(\alpha\setto \alpha;~\eta\setto D/\sqrt{3};~\hat{s}_1\setto \hat{s}_{t-1})$ (Algorithm~\ref{alg:ogd}), and set weight $w_{t,t} = 0$ \label{line:cbce-expert}
    
    Compute active set $\mathrm{Active}(t) = \{i \in [T] : t - L(i) < i \le t\}$, where $L(i)$ is defined in~\eqref{eq:lifetime}

    Compute prior probability $\pi_i \propto i^{-2}(1 + \lfloor \log_2 i \rfloor)^{-1}\one[i \in \mathrm{Active}(t)]$~\label{line:cbce_prior}

    Compute un-normalized probability $\hat{p}_i = \pi_{i} [w_{t, i}]_+$ for all $i \in [t]$~\label{line:cbce_unnormalized}

    Normalize $p = \hat{p} / \norm{\hat{p}}_1 \in \Delta^t$ if $\norm{\hat{p}}_1 > 0$, else $p = \pi$~\label{line:cbce_normalized}

    Compute predicted radius $\hat{s}_t = \sum_{i \in \mathrm{Active}(t)} p_{i} \hat{s}_{i,t}$ (for $t\ge 2$), and $\hat{s}_t = 0$ for $t = 1$~\label{line:cbce_aggregation}
    
    Observe input $X_t\in\cX$ and \textbf{return} prediction set $\hat{C}_t(X_t, \hat{s}_t)$
    
    \tcp{Use meta loss and per-expert losses to update experts}
    Observe true label $Y_t\in \cY$, compute true radius $S_t = \inf\{s\in\R: Y_t\in\hat{C}_t(X_t, s)\}$ and loss function $\ell^{(t)}(\cdot) = \ell_{1-\alpha}(S_t, \cdot)$

    \For{$i \in \mathrm{Active}(t)$\label{line:cbce_for}}{
        Update expert $\mathcal{A}_i$ with $(X_t, Y_t)$ and obtain next predicted radius $\hat{s}_{i,t+1}$ 

        Compute %
        $g_{i,t} = \begin{cases} 
            \frac{1}{D}\big(\ell^{(t)}(\hat{s}_t) - \ell^{(t)}(\hat{s}_{i,t})\big) & w_{i,t} > 0 \\
            \frac{1}{D}\big[\ell^{(t)}(\hat{s}_t) - \ell^{(t)}(\hat{s}_{i,t})\big]_+ & w_{i,t} \le 0 \\
        \end{cases}$~~~\label{line:cbce_gt}

        Update expert weight $w_{i,t+1} = \frac{1}{t - i + 1} \qty(\sum_{j=i}^{t} g_{i,j}) \qty( 1 + \sum_{j=i}^{t} w_{i,j} g_{i,j})$~\label{line:cbce_end}
    }
}
\caption{Strongly Adaptive Online Conformal Prediction (\method{}), adapted from \citet{jun2017cbce}.}
\label{alg:cbce}
\end{algorithm*}

\paragraph{Coverage and regret guarantees}

\citet{gibbs2021aci} show that ACI\footnote{Their results are established on the specific choice of $\cC_t$ in~\eqref{eq:aci_ct}, but can be extended directly to any $\cC_t$ that are nested sets.} achieves approximately valid (empirical) coverage in the sense that the empirical miscoverage frequency is close to the target level $\alpha$:
\begin{align}
\label{eq:aci_coverage}
\textstyle
\Err(T) \defeq \abs{ \frac{1}{T}\sum_{t=1}^T \err_t - \alpha } \le \frac{D + \eta}{\eta T}.
\end{align}
In addition to coverage, by standard online learning analyses, ACI achieves a \emph{regret} bound on the quantile losses $\{\ell^{(t)}\}_{t\in[T]}$: we have $\Reg(T) \le \cO(D^2/\eta+\eta T) \le \cO(D\sqrt{T})$ (with optimally chosen $\eta$), where
\begin{align}
\label{eq:aci_regret}
\textstyle
\Reg(T) \defeq \sum_{t=1}^T \ell^{(t)}(\hat{s}_t) - \inf_{s^\star\in\R} \sum_{t=1}^T \ell^{(t)}(s^\star).
\end{align}
One advantage of the regret as an additional performance measure alongside coverage is that it rules out certain algorithms that achieve good coverage in a ``trivial'' fashion and are not useful in practice --- for example, $\hat{C}_t$ may simply alternate between the empty set and the full set $\{\alpha,1-\alpha\}$ proportion of the time, which satisfies the coverage bound~\eqref{eq:aci_coverage} on arbitrary data distributions, but suffers from linear regret even on certain simple data distributions (cf. Appendix~\ref{appendix:trivial_alg}).

We remark that the regret has a connection to the coverage error in that $\Err(T)=|\frac{1}{T}\sum_{t=1}^T \grad \ell^{(t)}(\hat{s}_t)|$, i.e.\ the coverage error is equal to the average gradient (derivative) of the losses. However, without further distributional assumptions, regret bounds and coverage bounds do not imply each other in a black-box fashion (see e.g.~\citet[Appendix A]{gibbs2021aci}) and need to be established separately for each algorithm.

\section{Strongly Adaptive Online Conformal Prediction}
\label{section:saocp}

Our approach is motivated from the observation that regret minimization is in a sense limited, as the regret measures performance over the \emph{entire time horizon} $[T]$, which may be insufficient when the algorithm encounters changing environments. For example, if $S_t = 1$ for $1 \le t \le T / 2$ and $S_t = 100$ for $T / 2 < t \le T$, then achieving small regret on all \emph{(sub)-intervals} of size $T / 2$ is a much stronger guarantee than achieving small regret over $[T]$. For this reason, we seek {\em localized} guarantees over all intervals simultaneously, to prevent worst-case scenarios such as significant miscoverage or high radius within a specific interval.

The \emph{Strongly Adaptive Regret} (SARegret)~\citep{daniely2015saol,zhang2018dynamic} has been proposed in the online learning literature as a generalization of the regret that captures the performance of online learning algorithms over all intervals simultaneously. Concretely, for any $k\in[T]$, the SARegret of length $k$ of any algorithm is defined as
\begin{align}
\label{eq:sa_regret}
\SAReg(T,k) &\defeq \max_{[\tau,\tau+k-1]\subseteq [T]} \paren{ \sum_{t = \tau}^{\tau + k - 1} \ell^{(t)}(\hat{s}_t) - \inf_{s^\star} \sum_{t = \tau}^{\tau + k - 1} \ell^{(t)}(s^\star) }
\end{align}
$\SAReg(T,k)$ measures the maximum regret over all intervals of length $k$, which reduces to the usual regret at $k=T$, but may in addition be smaller for smaller $k$.

\paragraph{Algorithm: \method{}}
We leverage techniques for minimizing the strongly adaptive regret to perform online conformal prediction. Our main algorithm, Strongly Adaptive Online Conformal Prediction (\method{}, described in Algorithm~\ref{alg:cbce}), adapts the work of \citet{jun2017cbce} to the online conformal prediction setting.
At a high level, \method{} is a meta-algorithm that manages multiple \emph{experts}, where each expert is itself an arbitrary online learning algorithm taking charge of its own \emph{active interval} that has a finite \emph{lifetime}. At each $t\in[T]$, Algorithm~\ref{alg:cbce} instantiates a new expert $\mathcal{A}_t$ with active interval $[t,t+L(t)-1]$, where $L(t)$ is its lifetime:
\begin{align}
L(t) \defeq g \cdot \max_{n \in \Z}\{2^n : t \equiv 0 \text{ mod } 2^n \},
\label{eq:lifetime}
\end{align}
and $g\in\Z_{\ge 1}$ is a multiplier for the lifetime of each expert. It is straightforward to see that at most $g \lfloor \log_2 t \rfloor$ experts are active at any time $t$ under choice~\eqref{eq:lifetime}, granting Algorithm~\ref{alg:cbce} a total runtime of $\mathcal{O}(T \log T)$ for any $g = \Theta(1)$.
Then, at any time $t$, the predicted radius $\hat{s}_t$ is obtained by aggregating the predictions of active experts (Line~\ref{line:cbce_aggregation}):
\begin{align*}
\textstyle
    \hat{s}_t=\sum_{i\in\Active(t)} p_{i,t}\hat{s}_{i,t},
\end{align*}
where the weights $\{p_{i,t}\}_{i\in[t]}$ (Lines~\ref{line:cbce_prior}-\ref{line:cbce_normalized}) rely on the $\{w_{i,t}\}_{i\in[t]}$ computed by the \emph{coin betting} scheme~\citep{orabona2016coin,jun2017cbce} in Lines~\ref{line:cbce_gt}-\ref{line:cbce_end}.

\paragraph{Choice of expert}
In principle, \method{} allows any choice of the expert that is a good regret minimization algorithm over its own active interval satisfying \emph{anytime} regret guarantees. We choose the experts to be Scale-Free OGD (\methodBasic{}; Algorithm~\ref{alg:ogd})~\citep{orabona2018ogd}, a variant of OGD that decays its effective learning rate based on cumulative past gradient norms (cf.~\eqref{eq:ogd_update}). 

On the quantile loss~\eqref{eq:quantile_loss} (executed over the full horizon $[T]$ with learning rate $\eta = \Theta(D)$; $\eta = D / \sqrt{3}$ is optimal), \methodBasic{} enjoys an anytime regret guarantee (Proposition~\ref{prop:ogd_regret}) 
\begin{align}
\label{eq:ogd_regret}
\textstyle
    \Reg(t)\le \cO(D\sqrt{t})~~~\textrm{for all}~t\in[T],
\end{align}
which follows directly by applying~\citet[Theorem 2]{orabona2018ogd}. Plugging \methodBasic{} into Line~\ref{line:cbce-expert} of Algorithm~\ref{alg:cbce} gives our full \method{} algorithm.

\paragraph{\methodBasic{} as an independent algorithm}
As a strong regret minimization algorithm itself, \methodBasic{} can also be run independently (over $[T]$) as an algorithm for online conformal prediction (described in Algorithm~\ref{alg:ogd}). We find empirically that \methodBasic{} itself already achieves strong performances in many scenarios (Section~\ref{sec:experiments}).

\begin{algorithm}[t]
\small
\DontPrintSemicolon
\SetNoFillComment
\SetKwInOut{Input}{Input}
\newcommand\mycommfont[1]{\footnotesize\ttfamily\textcolor{blue}{#1}}
\SetCommentSty{mycommfont}
\Input{$\alpha\in(0,1)$, learning rate $\eta>0$, init. $\hat{s}_1\in\R$}

\For{$t \ge 1$}{
    Observe input $X_t\in\cX$

    \textbf{Return} prediction set $\hat{C}_t(X_t, \hat{s}_t)$

    Observe true label $Y_t\in\cY$ and compute true radius $S_t = \inf\{s\in\R: Y_t\in\hat{C}_t(X_t, s)\}$.
    
    Compute loss $\ell^{(t)}(\cdot) = \ell_{1-\alpha}(S_t, \cdot)$

    Update predicted radius
    \vspace{-1em}
    \begin{align}
    \label{eq:ogd_update}
    \hat{s}_{t+1} = \hat{s}_t - \eta \frac{\grad \ell^{(t)}(\hat{s}_t)}{\sqrt{\sum_{i=1}^{t} \norm{\grad \ell^{(i)}(\hat{s}_i)}_2^2}}
    \end{align}
    \vspace{-1em}
}
\caption{Scale-Free Online Gradient Descent (\methodBasic{}), adapted from \citet{orabona2018ogd}.}
\label{alg:ogd}
\end{algorithm}

\section{Theory}
\label{sec:theory}

\subsection{Strongly Adaptive Regret}
\label{sec:saregret}

We begin by showing the SARegret guarantee of \method{}. As we instantiate \method{} with \methodBasic{} as the experts, the proof follows directly by plugging the regret bound for \methodBasic{}~\eqref{eq:ogd_regret} into the SARegret guarantee for \method{}~\citep{jun2017cbce}, and can be found in Appendix~\ref{appendix:proof_saregret}.
\begin{proposition}[SARegret bound for \method{}]
\label{prop:saregret}\label{thm:cbce}
Algorithm~\ref{alg:cbce} achieves the following SARegret bound \emph{simultaneously for all lengths $k\in[T]$}:%
\begin{align}
\label{eq:cbce_saregret}
\SAReg(T, k) \le 15 D \sqrt{k(\log T + 1)} \le \tO(D\sqrt{k}).
\end{align}    
\end{proposition}
The $\tO(D\sqrt{k})$ rate achieved by \method{} is near-optimal for general online convex optimization problems, due to the standard regret lower bound $\Omega(D\sqrt{k})$ over any fixed interval of length $k$~\citep[Theorem 5.1]{orabona2019modern}.

\paragraph{Comparison with FACI}
The SARegret guarantee of \method{} in~\eqref{eq:cbce_saregret} improves substantially over the FACI (Fully Adaptive Conformal Inference) algorithm~\citep{gibbs2022faci}, an extension of ACI. 
Concretely, \eqref{eq:cbce_saregret} holds \emph{simultaneously for all lengths $k$}. By contrast, FACI achieves $\SAReg(T, k) \le \tO(D^2/\eta + \eta k)$ in our setting (cf. their Theorem 3.2), where $\eta>0$ is their meta-algorithm learning rate. This can imply the same rate $\tO(D\sqrt{k})$ for a \emph{single $k$} by optimizing $\eta$, but not multiple values of $k$ simultaneously.

Also, in terms of algorithm styles, while both \method{} and FACI are meta-algorithms that maintain multiple experts (base algorithms), a main difference between them is that all experts in FACI differ in their learning rates and are all active over $[T]$, whereas experts in \method{} differ in their active intervals (cf.~\eqref{eq:lifetime}).

\paragraph{Dynamic regret}
The dynamic regret---which measures the performance of an online learning algorithm against an arbitrary sequence of comparators---is another generalization of the regret for capturing the performance in changing environments~\citep{zinkevich2003olo,besbes2015non}. Building on the reduction from dynamic regret to strongly adaptive regret~\citep{zhang2018dynamic}, we show that \method{} also achieves near-optimal dynamic regret with respect to the optimal comparators on any interval within $[T]$, with rate depending on a certain \emph{path length} of the true radii $\set{S_t}_{t\ge 1}$; see Proposition~\ref{prop:cbce_dynamic_regret} and the discussions thererafter.

\subsection{Coverage}
\label{sec:coverage}

Recall the empirical coverage error defined in~\eqref{eq:aci_coverage}:
\begin{align*}
\textstyle
\Err(T) = \abs{\frac{1}{T}\sum_{t=1}^T \err_t - \alpha}.
\end{align*}    
Without any distributional assumptions, we show that \methodBasic{} achieves $\Err(t) \le \mathcal{O}(t^{-1/4}\log t)$ for any $t\in[T]$.
So its empirical coverage converges to the target $1 - \alpha$ as $T \to \infty$, similar to ACI (though with a slightly slower rate). The proof (Appendix~\ref{appendix:proof_coverage_ogd}) builds on a grouping argument and the fact that the effective learning rate $\eta/\sqrt{\sum_{\tau=1}^t\ltwo{\grad\elltau(\hat{s}_\tau)}^2}$ of \methodBasic{} changes slowly in $t$.
\begin{theorem}[Coverage bound for \methodBasic{}]
    \label{thm:ogd_coverage}
    Algorithm~\ref{alg:ogd} with any learning rate $\eta = \Theta(D)$ and any initialization $\hat{s}_1\in[0,D]$ achieves $\Err(T) \le \mathcal{O}(\alpha^{-2} T^{-1/4}\log T)$ for any $T\ge 1$.
\end{theorem}

We now provide a distribution-free coverage bound for \method{}, building on a similar grouping argument as in Theorem~\ref{thm:ogd_coverage}.
\begin{theorem}[Coverage bound for \method{}; Informal version of Theorem~\ref{thm:cbce_coverage_nonexp_full}]
\label{thm:cbce_coverage_nonexp}
    For any $T \ge 1$, a randomized variant of Algorithm~\ref{alg:cbce} where Line~\ref{line:cbce_aggregation} is replaced by sampling an expert $i\sim p_{t}$ and predicting $\hat{s}_t\defeq \hat{s}_{t,i}$ achieves
    \begin{align}
    \label{eq:coverage_cbce} \textstyle
    \Err(T) \le \cO\big( \inf_{\beta} ( T^{1/2-\beta} + T^{\beta-1} S_\beta(T))\big).
    \end{align}
\end{theorem}
Theorem~\ref{thm:cbce_coverage_nonexp} considers a randomized variant of \method{}, and its coverage bound depends on a quantity $S_\beta(T)$ (full definition in Theorem~\ref{thm:cbce_coverage_nonexp_full}) that measures the smoothness of the expert weights and the cumulative gradient norms for each individual expert. Both are expected for technical reasons and also appear in coverage bounds for other expert-style meta-algorithms such as FACI~\citep{gibbs2022faci}. For instance, if there exists $\beta\in(1/2,1)$ so that $S_\beta(T)\le \tO(T^\gamma)$ for some $\gamma<1-\beta$, then~\eqref{eq:coverage_cbce} implies a coverage bound $\Err(T)\le \tO(T^{-\min\{1/2-\beta,\beta-1+\gamma\}})=o_T(1)$.

We remark that Theorem~\ref{thm:cbce_coverage_nonexp} also holds more generically for other choices of the expert weights $\set{p_t}_{t\in[T]}$ (Line~\ref{line:cbce_unnormalized}-\ref{line:cbce_normalized}) and active intervals, not just those specified in Algorithm~\ref{alg:cbce}. In particular, \methodBasic{} is the special case where there is only a single active expert over $[T]$. In this case, we can recover the $\tO(\alpha^{-2}T^{-1/4})$ bound of Theorem~\ref{thm:ogd_coverage} (see Appendix~\ref{appendix:subsume_discussion} for a detailed discussion).

\paragraph{Additional coverage guarantee under distributional assumptions} Under some mild regularity assumptions on the distributions of $S_1, \ldots, S_T$, we show in Theorem~\ref{thm:cbce_coverage} that \method{} achieves approximately valid coverage on {\em every} sub-interval of time. Its coverage error on any interval $I=[\tau, \tau + k - 1] \subseteq [1, T]$ is $\tO(k^{-1/(2q)}+({\rm Var}_I/k)^{1/q})$ for a certain $q\ge 2$ that quantifies the regularity of the distribution, and ${\rm Var}_I$ is a certain notion of variation between the true quantiles of $S_t$ over $t\in I$ (cf.~\eqref{eqn:interval_variation}). In particular, we obtain an approximately valid coverage on any interval $I$ for which ${\rm Var}_I=o(k)$.

\setlength{\tabcolsep}{4pt}
\AtBeginEnvironment{tabular}{\footnotesize}

\section{Experiments}
\label{sec:experiments}

We test \methodBasic{} (Algorithm~\ref{alg:ogd}) and \method{} (Algorithm~\ref{alg:cbce}) empirically on two representative real-world online uncertainty quantification tasks: time series forecasting (Section~\ref{sec:exp_time_series}) and image classification under distribution shift (Section~\ref{sec:exp_image}). Choices of the prediction sets $\{\hat{C}_t(x,s)\}_{x,s}$ will be described within each experiment. 
In both experiments, we compare against the following methods:
\begin{enumerate}[topsep=-2pt,itemsep=0pt,parsep=2pt]
\item \scp{}: standard Split Conformal Prediction \citep{vovk2005alrw} adapted to the online setting, which simply predicts the $(1-\alpha)$-quantile of the past radii. \scp{} does not admit a valid coverage guarantee in our settings as the data may not be exchangeable in general;
\item \nexcp{}: Non-Exchangeable SCP \citep{barber2022nexcp}, a variant of SCP that handles non-exchangeable data by reweighting. We follow their recommendations and use an exponential weighting scheme that upweights more recent observations;
\item \faci{}~\citep{gibbs2022faci} with their specific ``quantile parametrization''~\eqref{eq:aci_ct}, and score function $\wt{S}_t$ corresponding to our choice of $\hat{C}_t$;
\item \facis{}: Generalized version of FACI applied to predicting the radius $\hat{s}_t$'s on our choice of $\hat{C}_t$ directly.
\end{enumerate}
Additional details about all methods can be found in Appendix~\ref{appendix:exp_details}. Throughout this section we choose the target coverage level to be the standard $1-\alpha=90\%$.

\subsection{Time Series Forecasting}
\label{sec:exp_time_series}

\paragraph{Setup}
We consider multi-horizon time series forecasting problems with real-valued observations $\{y_t\}_{t\ge 1}\in\R$, where the base predictor $\hat{f}$ uses the history $X_t\defeq y_{1:t}$ to predict $H$ steps into the future, i.e. $\hat{f}(X_t)=\{\hat{f}^{(h)}(X_t)\}_{h\in[H]} =\{\hat{y}_{t+h}^{(h)}\}_{h\in[H]}$, where $\hat{y}_{t+h}^{(h)}$ is a prediction for $y_{t+h}$. Using $\hat{f}(X_t)$, we produce fixed-width prediction intervals
\begin{align}
\label{eq:time_series_ct}
\hat{C}_{t}^{(h)}(X_t, \hat{s}_{t}^{(h)}) \defeq \big[ \hat{y}^{(h)}_{t+h} - \hat{s}_{t}^{(h)}, \hat{y}^{(h)}_{t+h} + \hat{s}_{t}^{(h)} \big],
\end{align}
where $\hat{s}_{t}^{(h)}$ is predicted by an independent copy of the online conformal prediction algorithm for each $h\in[H]$ (so that there are $H$ such algorithms in parallel). We form our online setting using a standard rolling window evaluation loop, wherein each \emph{batch} consists of predicting all $H$ intervals $\{\hat{C}_t^{(h)}\}_{h\in[H]}$, observing all $H$ true values $\{y_{t+h}\}_{h\in[H]}$, and moving to the next batch by setting $t\to t+H$. For each $h \in [H]$, we only evaluate $y_{t+h}$ against one interval $\hat{C}^{(h)}_t(X_t, \hat{s}^{(h)}_t)$. After the evaluation is done, we use all pairs $\{(y_{t+k}, \hat{y}_{t+k}^{(h)})\}_{k \in [H]}$ to update $\hat{s}_{t}^{(h)} \to \hat{s}^{(h)}_{t+H}$.

\paragraph{Base predictors}
We consider three diverse types of base predictors (models), and we use their implementations in Merlion v2.0.0~\citep{bhatnagar2021merlion}:%

\begin{enumerate}[topsep=-2pt,itemsep=0pt,parsep=2pt]
\item LGBM: A model which uses gradient boosted trees to predict $\hat{y}_{t+h}^{(h)} = \hat{f}^{(h)}(y_{t-L+1}, \ldots, y_t)$. This approach attains strong performance on many time series benchmarks~\citep{elsayed2021gbtforecast,bhatnagar2021merlion}.
\item $\mathrm{ARIMA}(10, d^\star, 10)$: The classical AutoRegressive Integrated Moving Average stochastic process model for a time series, where the difference order $d^\star$ is chosen by KPSS stationarity test \citep{kpss1992}.
\item Prophet \citep{taylor2017prophet}: A popular Bayesian model which directly predicts the value $y$ as a function of time, i.e.\ $\hat{y}_t = \hat{f}(t)$.
\end{enumerate}

\paragraph{Datasets}
We evaluate on four datasets totaling 5111 time series: the hourly (414 time series), daily (4227 time series), and weekly (359 time series) subsets of the M4 Competition, a dataset of time series from many domains including industries, demographics, environment, finance, and transportation \citep{makridakis2018m4}; and NN5, a dataset of 111 time series of daily banking data \citep{bentaieb2012nn5}. We normalize each time series to lie in $[0, 1]$. 

We use horizons $H$ of 24, 30, and 26 for hourly, daily, and weekly data, respectively. Each time series of length $L$ is split into a training set of length $L-120$ with 80\% for training the base predictor and 20\% for initializing the UQ methods, and a test set of length $120$ to test the UQ methods.

\paragraph{Metrics}
For each experiment, we average the following statistics across all time series: global coverage, median width, worst-case local coverage error
\begin{align}
    \label{eq:sa_err}
    \mathrm{LCE}_k \defeq \max_{[\tau, \tau + k - 1] \subseteq [1, T]} \textstyle \abs{\alpha - \frac{1}{k} \sum_{t=\tau}^{\tau+k-1} \err_t},
\end{align} and strongly adaptive regret $\SAReg(T, k)$~\eqref{eq:sa_regret}, which we abbreviate as $\SAReg_k$. In all cases, we use an interval length of $k = 20$. We also report the average mean absolute error (MAE) of each base predictor.

\begin{table*}[t]
\centering
\begin{tabular}{l|cccc|cccc|cccc}
\toprule
 & \multicolumn{4}{c|}{LGBM (MAE = 0.06)} & \multicolumn{4}{c|}{ARIMA (MAE = 0.18)} & \multicolumn{4}{c}{Prophet (MAE = 0.12)} \\
Method & Coverage & Width & $\mathrm{LCE}_k$ & $\SAReg_k$ & Coverage & Width & $\mathrm{LCE}_k$ & $\SAReg_k$ & Coverage & Width & $\mathrm{LCE}_k$ & $\SAReg_k$ \\
\midrule
SCP & \color{red} 0.844 & 0.127 & 0.252 & 0.017 & \color{ForestGreen} 0.871 & 0.245 & 0.237 & 0.039 & \color{red} 0.783 & 0.178 & 0.355 & 0.019 \\
NExCP & \color{ForestGreen} 0.875 & 0.134 & 0.197 & 0.013 & \color{ForestGreen} 0.871 & 0.245 & 0.227 & 0.040 & \color{ForestGreen} 0.856 & 0.187 & 0.231 & 0.010 \\
FACI & \color{ForestGreen} 0.866 & \bfseries 0.113 & 0.180 & \bfseries 0.009 & \color{ForestGreen} 0.866 & \underline{0.232} & 0.214 & \underline{0.034} & \color{ForestGreen} 0.867 & \underline{0.175} & 0.184 & \underline{0.006} \\
\methodBasic{} & \color{ForestGreen} 0.889 & 0.138 & \underline{0.154} & 0.011 & \color{ForestGreen} 0.877 & 0.250 & \underline{0.195} & 0.037 & \color{ForestGreen} 0.888 & 0.186 & \bfseries 0.138 & 0.007 \\
FACI-S & \color{ForestGreen} 0.883 & 0.128 & 0.163 & 0.010 & \color{ForestGreen} 0.872 & 0.238 & 0.201 & 0.035 & \color{ForestGreen} 0.885 & 0.180 & 0.144 & \underline{0.006} \\
\method{} & \color{ForestGreen} 0.882 & \underline{0.121} & \bfseries 0.143 & \bfseries 0.009 & \color{ForestGreen} 0.864 & \bfseries 0.221 & \bfseries 0.190 & \bfseries 0.033 & \color{ForestGreen} 0.872 & \bfseries 0.173 & \underline{0.143} & \bfseries 0.005 \\
\bottomrule
\end{tabular}
\caption{Results on M4 Hourly (414 time series) with target coverage $1 - \alpha = 0.9$ and interval size $k = 20$. Best results are {\bfseries bold}, while second best are \underline{underlined}, as long as the method's global coverage is in $(0.85, 0.95)$ (green). For all base predictors, \method{} achieves the best or second-best width, local coverage error, and strongly adaptive regret.
}
\label{tab:m4_hourly}
\end{table*}

\begin{table*}
    \centering
\begin{tabular}{l|cccc|cccc|cccc}
\toprule
 & \multicolumn{4}{c|}{LGBM (MAE = 0.14)} & \multicolumn{4}{c|}{ARIMA (MAE = 0.06)} & \multicolumn{4}{c}{Prophet (MAE = 0.32)} \\
Method & Coverage & Width & $\mathrm{LCE}_k$ & $\SAReg_k$ & Coverage & Width & $\mathrm{LCE}_k$ & $\SAReg_k$ & Coverage & Width & $\mathrm{LCE}_k$ & $\SAReg_k$ \\
\midrule
SCP & \color{red} 0.769 & 0.184 & 0.466 & 0.031 & \color{ForestGreen} 0.896 & 0.122 & 0.290 & 0.018 & \color{red} 0.599 & 0.349 & 0.614 & 0.051 \\
NExCP & \color{red} 0.818 & 0.183 & 0.420 & 0.015 & \color{ForestGreen} 0.891 & 0.116 & 0.296 & 0.012 & \color{red} 0.715 & 0.356 & 0.559 & 0.019 \\
FACI & \color{red} 0.846 & 0.169 & 0.308 & 0.008 & \color{ForestGreen} 0.886 & \underline{0.101} & 0.259 & \bfseries 0.008 & \color{red} 0.767 & 0.344 & 0.397 & 0.014 \\
\methodBasic{} & \color{ForestGreen} 0.873 & 0.173 & 0.246 & 0.011 & \color{ForestGreen} 0.892 & 0.106 & 0.245 & 0.011 & \color{ForestGreen} 0.862 & 0.354 & 0.220 & 0.008 \\
FACI-S & \color{ForestGreen} 0.875 & \underline{0.169} & \underline{0.240} & \underline{0.010} & \color{ForestGreen} 0.891 & 0.103 & \underline{0.243} & 0.010 & \color{ForestGreen} 0.866 & \underline{0.352} & \underline{0.210} & \underline{0.007} \\
\method{} & \color{ForestGreen} 0.869 & \bfseries 0.162 & \bfseries 0.213 & \bfseries 0.007 & \color{ForestGreen} 0.875 & \bfseries 0.093 & \bfseries 0.238 & \bfseries 0.008 & \color{ForestGreen} 0.867 & \bfseries 0.349 & \bfseries 0.172 & \bfseries 0.005 \\
\bottomrule
\end{tabular}
\caption{Results on M4 Daily (4227 time series) with target coverage $1 - \alpha = 0.9$ and interval size $k = 20$. Best results are {\bfseries bold}, while second best are \underline{underlined}, as long as the method's global coverage is in $(0.85, 0.95)$ (green). \method{} achieves the best width, local coverage error, and strongly adaptive regret for all base predictors. The only methods which achieve global coverage in $(0.85, 0.95)$ for LGBM and Prophet are the ones that predict $\hat{s}_{t+1}$ directly, not as a quantile of $S_1, \ldots, S_t$.}
\label{tab:m4_daily}
\end{table*}

\paragraph{Results}
We report results on M4 Hourly and M4 Daily in Tables \ref{tab:m4_hourly}, \ref{tab:m4_daily}, and on M4 Weekly and NN5 Daily in Tables \ref{tab:m4_weekly}, and \ref{tab:nn5_daily} (Appendix~\ref{appendix:more_ts}).

\method{} consistently achieves global coverage in $(0.85, 0.95)$, and it obtains the best or second-best interval width, local coverage error, and strongly adaptive regret for all base predictors on all 3 M4 datasets. FACI-S generally achieves better $\mathrm{LCE}_k$ and $\SAReg_k$ than FACI, showing the benefits of predicting $\hat{s}_{t+1}$ directly, rather than as a quantile of $S_1, \ldots, S_t$. The relative performance of FACI-S and \methodBasic{} varies, though FACI-S is usually a bit better. However, \method{} consistently achieves better $\mathrm{LCE}_k$ and $\SAReg_k$ than both FACI-S and \methodBasic{}.

There are multiple instances where all of \scp{}/\nexcp{}/\faci{} fail to attain global coverage in $(0.85, 0.95)$ (Tables~\ref{tab:m4_daily} and \ref{tab:m4_weekly}). The base predictor's MAE is at least 0.14 in all these cases, suggesting an advantage of predicting $\hat{s}_{t+1}$ directly as in \methodBasic{}/\method{} when the underlying base predictor is inaccurate.

\paragraph{Additional experiments with ensemble models} 
In Appendix~\ref{appendix:ensemble}, we use EnbPI \citep{xu2021enbpi} to train a bootstrapped ensemble, and we compare EnbPI's results with those obtained by applying NExCP, FACI, \methodBasic{}, and \method{} to the residuals produced by that ensemble. The results largely mirror those in the main paper. %

\begin{figure*}[t]
    \centering
    \includegraphics[width=0.9\textwidth,trim={0 0.7cm 0 0}]{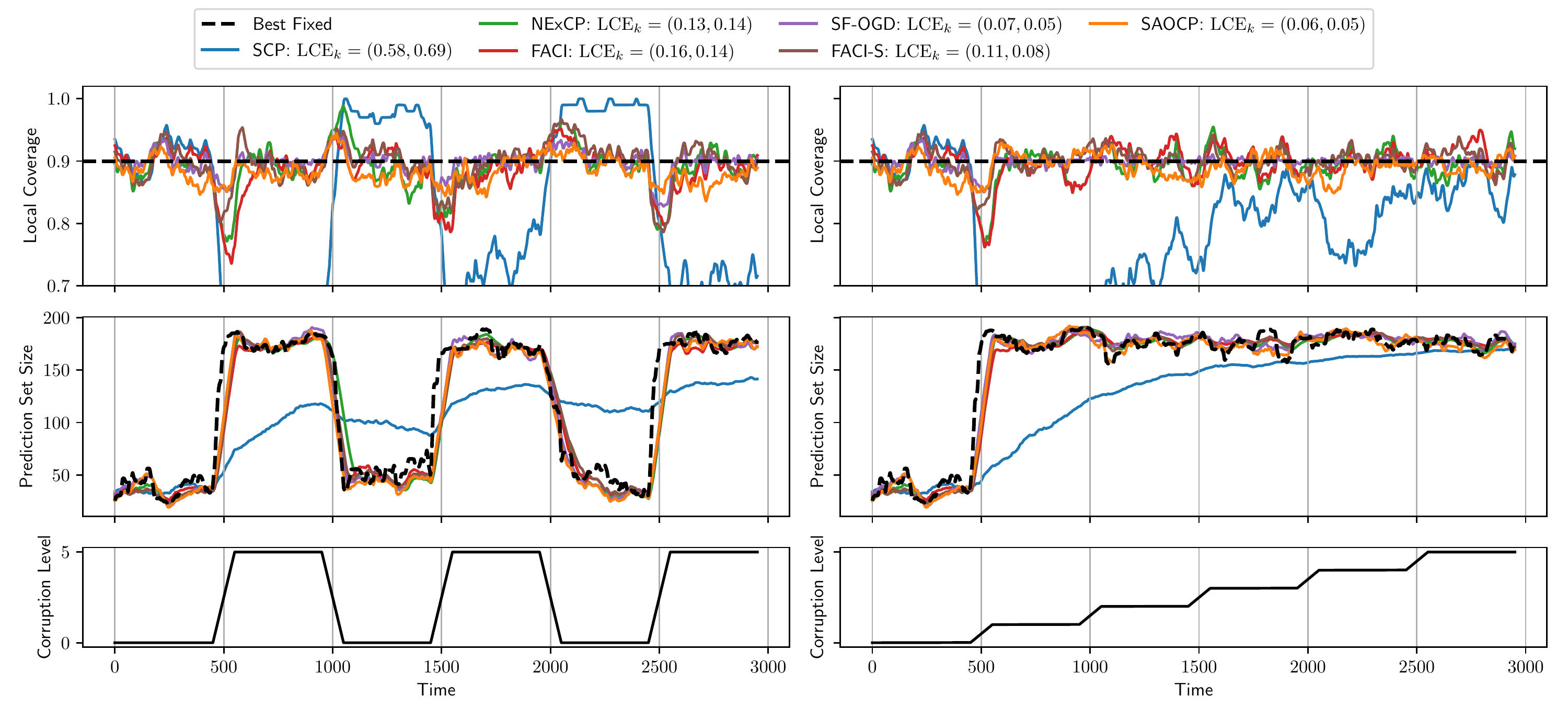}
    \caption{Local coverage (top row) and prediction set size (second row) achieved by various UQ methods when the distribution shifts between TinyImageNet and TinyImageNet-C every 500 steps. We plot moving averages with a window size of 100. Left: sudden shifts between corruption level 0 and 5. Right: gradual shift from level 0 to 5. \method{} and \methodBasic{}'s local coverage remain the closest to the target of 0.9, especially at the change points. While their prediction sets have similar size, $\mathrm{LCE}_k$ is lower for \method{} than \methodBasic{}.}
    \label{fig:tinyimagenet}
\end{figure*}

\subsection{Image Classification Under Distribution Shift}
\label{sec:exp_image}
\paragraph{Datasets and Setup} We evaluate the ability of conformal prediction methods to maintain coverage when the underlying distribution shifts in a systematic manner. We use a ResNet-50 classifier \citep{he2015resnet} pre-trained on ImageNet and implemented in PyTorch \citep{pytorch}. Here, $x \in \cX$ is an image, and $y \in \cY  = [m]$ is its class. To construct structured distribution shifts away from the training distribution, we use TinyImageNet-C and ImageNet-C \citep{hendrycks2019benchmarking}, which are corrupted versions of the TinyImageNet ($m = 200$ classes) \citep{Le2015TinyIV} and ImageNet ($m = 1000$ classes) \citep{deng2009imagenet} test sets designed to evaluate model robustness. These corrupted datasets apply 15 visual corruptions at 5 different severity levels to each image in the original test set. 

We consider two regimes: {\bf sudden shifts} where the corruption level alternates between 0 (the base test set) and 5, and {\bf gradual shifts} where the corruption level increases in the order of $\set{0,1,\dots,5}$. We randomly sample 500 data points for each corruption level before changing to the next level. 

\paragraph{Prediction Sets} We follow \citet{angelopoulos2021uncertainty} to construct our prediction sets. Let $\hat{f}: \R^d \to \Delta^m$ be a classifier that outputs a probability distribution on the $m$-simplex. At each $t$, we sample $U_t \sim \mathrm{Unif}[0, 1]$ and let
\begin{align}
    \textstyle \notag
    S_t(x, y) &= \textstyle \lambda \sqrt{[k_y - k_{reg}]_+} + U_t \hat{f}_y(x) + \sum_{i=1}^{k_y - 1} \hat{f}_{\pi(i)}(x) \\
    \label{eq:raps}
    \hat{C}_t(X_t) &= \{ y: S_t(X_t, y) \le \hat{s}_t \},
\end{align}
where $\pi$ is the permutation that ranks $\hat{f}(x)$ in decreasing order, $\pi(k_y) = y$, and $\lambda$ and $k_{reg}$ are regularization parameters designed to reduce the size of the prediction set. For TinyImageNet, we use $\lambda = 0.01$ and $k_{reg} = 20$. For ImageNet, we use $\lambda = 0.01$ and $k_{reg} = 10$.

\paragraph{Metrics} When evaluating the UQ methods, we plot the local coverage and prediction set size (PSS) of each method using an interval length of $k = 100$,
\begin{align*}
    \mathrm{LocalCov}(t) &= \textstyle \frac{1}{100} \sum_{i=t}^{t+99} \one[Y_i \in \hat{C}_i(X_i)] \\
    \mathrm{LocalPSS}(t) &= \textstyle \frac{1}{100}\sum_{i=t}^{t+99} \abs{\hat{C}_i(X_i)}.
\end{align*}
We compare the local coverage to a target of $1 - \alpha$, while we compare the local PSS to the $1 - \alpha$ empirical quantile of the oracle set sizes $\mathrm{PSS}_t^\star = \abs{\{y: S_t(X_t, y) \le S_t(X_t, Y_t)\}}$. These targets are the ``best fixed'' values in each window. We also report the worst-case local coverage error $\mathrm{LCE}_{100}$~\eqref{eq:sa_err}.

\paragraph{Results} We evaluate the UQ methods on TinyImageNet and TinyImageNet-C in Figure~\ref{fig:tinyimagenet}, and on ImageNet and ImageNet-C in Figure~\ref{fig:imagenet} (Appendix~\ref{appendix:more_cv}). In both sudden and gradual distribution shift, the local coverage of \method{} and \methodBasic{} remains the closest to the target of 0.9. The difference is more notable when the distribution shifts suddenly. When the distribution shifts more gradually, NExCP, FACI, and FACI-S have worse coverage than \method{} and \methodBasic{} at the first change point, which is where the largest change in the best set size occurs. 

All methods besides SCP predict sets of similar sizes, though FACI's, FACI-S's, and NExCP's prediction set sizes adapt more slowly to changes in the best fixed size (e.g. $t \in [500, 700]$ for gradual shift in Figure~\ref{fig:imagenet}). On TinyImageNet, \method{} obtains slightly better local coverage than \methodBasic{}, and they both have similar prediction set sizes (Figure~\ref{fig:tinyimagenet}). On ImageNet, \method{} and \methodBasic{} attain similar local coverages, but \method{} tends to attain that coverage with a smaller prediction set (Figure~\ref{fig:imagenet}).

\vspace{-.5em}
\section{Conclusion}
This paper develops new algorithms for online conformal prediction under arbitrary distribution shifts. Our algorithms achieve approximately valid coverage and better strongly adaptive regret than existing work. On real-world experiments, our proposed algorithms achieve coverage closer to the target within local windows, and they produce smaller prediction sets than existing methods. Our work opens up many questions for future work, such as obtaining stronger coverage guarantees, or characterizing the optimality of the learned radii under various settings with distribution shift.

\bibliography{main}

\begin{thebibliography}{66}
\providecommand{\natexlab}[1]{#1}
\providecommand{\url}[1]{\texttt{#1}}
\expandafter\ifx\csname urlstyle\endcsname\relax
  \providecommand{\doi}[1]{doi: #1}\else
  \providecommand{\doi}{doi: \begingroup \urlstyle{rm}\Url}\fi

\bibitem[Angelopoulos et~al.(2021{\natexlab{a}})Angelopoulos, Bates,
  Cand{\`e}s, Jordan, and Lei]{angelopoulos2021learn}
Anastasios~N Angelopoulos, Stephen Bates, Emmanuel~J Cand{\`e}s, Michael~I
  Jordan, and Lihua Lei.
\newblock Learn then test: Calibrating predictive algorithms to achieve risk
  control.
\newblock \emph{arXiv preprint arXiv:2110.01052}, 2021{\natexlab{a}}.

\bibitem[Angelopoulos et~al.(2022{\natexlab{a}})Angelopoulos, Bates, Fisch,
  Lei, and Schuster]{angelopoulos2022conformal}
Anastasios~N Angelopoulos, Stephen Bates, Adam Fisch, Lihua Lei, and Tal
  Schuster.
\newblock Conformal risk control.
\newblock \emph{arXiv preprint arXiv:2208.02814}, 2022{\natexlab{a}}.

\bibitem[Angelopoulos et~al.(2022{\natexlab{b}})Angelopoulos, Kohli, Bates,
  Jordan, Malik, Alshaabi, Upadhyayula, and Romano]{angelopoulos2022image}
Anastasios~N Angelopoulos, Amit~Pal Kohli, Stephen Bates, Michael Jordan,
  Jitendra Malik, Thayer Alshaabi, Srigokul Upadhyayula, and Yaniv Romano.
\newblock Image-to-image regression with distribution-free uncertainty
  quantification and applications in imaging.
\newblock In \emph{International Conference on Machine Learning}, pages
  717--730. PMLR, 2022{\natexlab{b}}.

\bibitem[Angelopoulos and Bates(2021)]{angelopoulos-gentle}
Anastasios~Nikolas Angelopoulos and Stephen Bates.
\newblock A gentle introduction to conformal prediction and distribution-free
  uncertainty quantification, 2021.
\newblock URL \url{https://arxiv.org/abs/2107.07511}.

\bibitem[Angelopoulos et~al.(2021{\natexlab{b}})Angelopoulos, Bates, Jordan,
  and Malik]{angelopoulos2021uncertainty}
Anastasios~Nikolas Angelopoulos, Stephen Bates, Michael Jordan, and Jitendra
  Malik.
\newblock Uncertainty sets for image classifiers using conformal prediction.
\newblock In \emph{International Conference on Learning Representations},
  2021{\natexlab{b}}.
\newblock URL \url{https://openreview.net/forum?id=eNdiU_DbM9}.

\bibitem[Bai et~al.(2022)Bai, Mei, Wang, Zhou, and Xiong]{bai2022efficient}
Yu~Bai, Song Mei, Huan Wang, Yingbo Zhou, and Caiming Xiong.
\newblock Efficient and differentiable conformal prediction with general
  function classes.
\newblock In \emph{International Conference on Learning Representations}, 2022.
\newblock URL \url{https://openreview.net/forum?id=Ht85_jyihxp}.

\bibitem[Barber et~al.(2021)Barber, Cand{\`e}s, Ramdas, and
  Tibshirani]{barber2021jackknife}
Rina~Foygel Barber, Emmanuel~J. Cand{\`e}s, Aaditya Ramdas, and Ryan~J.
  Tibshirani.
\newblock {Predictive inference with the jackknife+}.
\newblock \emph{The Annals of Statistics}, 49\penalty0 (1):\penalty0 486 --
  507, 2021.
\newblock \doi{10.1214/20-AOS1965}.
\newblock URL \url{https://doi.org/10.1214/20-AOS1965}.

\bibitem[Barber et~al.(2022)Barber, Candès, Ramdas, and
  Tibshirani]{barber2022nexcp}
Rina~Foygel Barber, Emmanuel~J. Candès, Aaditya Ramdas, and Ryan~J.
  Tibshirani.
\newblock Conformal prediction beyond exchangeability, 2022.
\newblock URL \url{https://arxiv.org/abs/2202.13415}.

\bibitem[Bates et~al.(2021)Bates, Angelopoulos, Lei, Malik, and
  Jordan]{bates2021distribution}
Stephen Bates, Anastasios Angelopoulos, Lihua Lei, Jitendra Malik, and Michael
  Jordan.
\newblock Distribution-free, risk-controlling prediction sets.
\newblock \emph{Journal of the ACM (JACM)}, 68\penalty0 (6):\penalty0 1--34,
  2021.

\bibitem[{Ben Taieb} et~al.(2012){Ben Taieb}, Bontempi, Atiya, and
  Sorjamaa]{bentaieb2012nn5}
Souhaib {Ben Taieb}, Gianluca Bontempi, Amir~F. Atiya, and Antti Sorjamaa.
\newblock A review and comparison of strategies for multi-step ahead time
  series forecasting based on the nn5 forecasting competition.
\newblock \emph{Expert Systems with Applications}, 39\penalty0 (8):\penalty0
  7067--7083, 2012.
\newblock ISSN 0957-4174.
\newblock \doi{https://doi.org/10.1016/j.eswa.2012.01.039}.
\newblock URL
  \url{https://www.sciencedirect.com/science/article/pii/S0957417412000528}.

\bibitem[Besbes et~al.(2015)Besbes, Gur, and Zeevi]{besbes2015non}
Omar Besbes, Yonatan Gur, and Assaf Zeevi.
\newblock Non-stationary stochastic optimization.
\newblock \emph{Operations research}, 63\penalty0 (5):\penalty0 1227--1244,
  2015.

\bibitem[Bhatnagar et~al.(2021)Bhatnagar, Kassianik, Liu, Lan, Yang, Cassius,
  Sahoo, Arpit, Subramanian, Woo, Saha, Jagota, Gopalakrishnan, Singh,
  Krithika, Maddineni, Cho, Zong, Zhou, Xiong, Savarese, Hoi, and
  Wang]{bhatnagar2021merlion}
Aadyot Bhatnagar, Paul Kassianik, Chenghao Liu, Tian Lan, Wenzhuo Yang, Rowan
  Cassius, Doyen Sahoo, Devansh Arpit, Sri Subramanian, Gerald Woo, Amrita
  Saha, Arun~Kumar Jagota, Gokulakrishnan Gopalakrishnan, Manpreet Singh, K~C
  Krithika, Sukumar Maddineni, Daeki Cho, Bo~Zong, Yingbo Zhou, Caiming Xiong,
  Silvio Savarese, Steven Hoi, and Huan Wang.
\newblock Merlion: A machine learning library for time series.
\newblock 2021.
\newblock URL \url{https://arxiv.org/abs/2109.09265}.

\bibitem[Candès et~al.(2021)Candès, Lei, and Ren]{candes2021survival}
Emmanuel~J. Candès, Lihua Lei, and Zhimei Ren.
\newblock Conformalized survival analysis, 2021.
\newblock URL \url{https://arxiv.org/abs/2103.09763}.

\bibitem[Cauchois et~al.(2020)Cauchois, Gupta, Ali, and
  Duchi]{cauchois2020robust}
Maxime Cauchois, Suyash Gupta, Alnur Ali, and John~C. Duchi.
\newblock Robust validation: Confident predictions even when distributions
  shift, 2020.
\newblock URL \url{https://arxiv.org/abs/2008.04267}.

\bibitem[Cauchois et~al.(2022)Cauchois, Gupta, and Duchi]{cauchois2022knowing}
Maxime Cauchois, Suyash Gupta, and John~C. Duchi.
\newblock Knowing what you know: Valid and validated confidence sets in
  multiclass and multilabel prediction.
\newblock \emph{J. Mach. Learn. Res.}, 22\penalty0 (1), jul 2022.
\newblock ISSN 1532-4435.

\bibitem[Chernozhukov et~al.(2018)Chernozhukov, W\"{u}thrich, and
  Yinchu]{chernozhukov18a}
Victor Chernozhukov, Kaspar W\"{u}thrich, and Zhu Yinchu.
\newblock Exact and robust conformal inference methods for predictive machine
  learning with dependent data.
\newblock In Sébastien Bubeck, Vianney Perchet, and Philippe Rigollet,
  editors, \emph{Proceedings of the 31st Conference On Learning Theory},
  volume~75 of \emph{Proceedings of Machine Learning Research}, pages 732--749.
  PMLR, 06--09 Jul 2018.
\newblock URL \url{https://proceedings.mlr.press/v75/chernozhukov18a.html}.

\bibitem[Daniely et~al.(2015)Daniely, Gonen, and
  Shalev-Shwartz]{daniely2015saol}
Amit Daniely, Alon Gonen, and Shai Shalev-Shwartz.
\newblock Strongly adaptive online learning.
\newblock In Francis Bach and David Blei, editors, \emph{Proceedings of the
  32nd International Conference on Machine Learning}, volume~37 of
  \emph{Proceedings of Machine Learning Research}, pages 1405--1411, Lille,
  France, 07--09 Jul 2015. PMLR.
\newblock URL \url{https://proceedings.mlr.press/v37/daniely15.html}.

\bibitem[Dashevskiy and Luo(2008)]{dashevskiy2008network}
Mikhail Dashevskiy and Zhiyuan Luo.
\newblock Network traffic demand prediction with confidence.
\newblock In \emph{IEEE GLOBECOM 2008 - 2008 IEEE Global Telecommunications
  Conference}, pages 1--5, 2008.
\newblock \doi{10.1109/GLOCOM.2008.ECP.284}.

\bibitem[Deng et~al.(2009)Deng, Dong, Socher, Li, Li, and
  Fei-Fei]{deng2009imagenet}
Jia Deng, Wei Dong, Richard Socher, Li-Jia Li, Kai Li, and Li~Fei-Fei.
\newblock Imagenet: A large-scale hierarchical image database.
\newblock In \emph{2009 IEEE conference on computer vision and pattern
  recognition}, pages 248--255. Ieee, 2009.

\bibitem[Elsayed et~al.(2021)Elsayed, Thyssens, Rashed, Jomaa, and
  Schmidt-Thieme]{elsayed2021gbtforecast}
Shereen Elsayed, Daniela Thyssens, Ahmed Rashed, Hadi~Samer Jomaa, and Lars
  Schmidt-Thieme.
\newblock Do we really need deep learning models for time series forecasting?,
  2021.
\newblock URL \url{https://arxiv.org/abs/2101.02118}.

\bibitem[Fannjiang et~al.(2022)Fannjiang, Bates, Angelopoulos, Listgarten, and
  Jordan]{fannjiang2022conformal}
Clara Fannjiang, Stephen Bates, Anastasios Angelopoulos, Jennifer Listgarten,
  and Michael~I Jordan.
\newblock Conformal prediction for the design problem.
\newblock \emph{arXiv preprint arXiv:2202.03613}, 2022.

\bibitem[Feldman et~al.(2022)Feldman, Ringel, Bates, and
  Romano]{feldman2022risk}
Shai Feldman, Liran Ringel, Stephen Bates, and Yaniv Romano.
\newblock Risk control for online learning models, 2022.
\newblock URL \url{https://arxiv.org/abs/2205.09095}.

\bibitem[Gibbs and Candès(2021)]{gibbs2021aci}
Isaac Gibbs and Emmanuel Candès.
\newblock Adaptive conformal inference under distribution shift.
\newblock In A.~Beygelzimer, Y.~Dauphin, P.~Liang, and J.~Wortman Vaughan,
  editors, \emph{Advances in Neural Information Processing Systems}, 2021.
\newblock URL \url{https://openreview.net/forum?id=6vaActvpcp3}.

\bibitem[Gibbs and Candès(2022)]{gibbs2022faci}
Isaac Gibbs and Emmanuel Candès.
\newblock Conformal inference for online prediction with arbitrary distribution
  shifts, 2022.
\newblock URL \url{https://arxiv.org/abs/2208.08401}.

\bibitem[Gupta et~al.(2019)Gupta, Kuchibhotla, and Ramdas]{gupta2019nested}
Chirag Gupta, Arun~K Kuchibhotla, and Aaditya~K Ramdas.
\newblock Nested conformal prediction and quantile out-of-bag ensemble methods.
\newblock \emph{arXiv preprint arXiv:1910.10562}, 2019.

\bibitem[Hazan(2022)]{hazan2022oco}
Elad Hazan.
\newblock \emph{Introduction to Online Convex Optimization}.
\newblock MIT Press, Cambridge, MA, USA, 2022.
\newblock ISBN 9780262046985.

\bibitem[He et~al.(2016)He, Zhang, Ren, and Sun]{he2015resnet}
Kaiming He, Xiangyu Zhang, Shaoqing Ren, and Jian Sun.
\newblock Deep residual learning for image recognition.
\newblock In \emph{2016 IEEE Conference on Computer Vision and Pattern
  Recognition (CVPR)}, pages 770--778, 2016.
\newblock \doi{10.1109/CVPR.2016.90}.

\bibitem[Hendrycks and Dietterich(2019)]{hendrycks2019benchmarking}
Dan Hendrycks and Thomas Dietterich.
\newblock Benchmarking neural network robustness to common corruptions and
  perturbations.
\newblock In \emph{International Conference on Learning Representations}, 2019.
\newblock URL \url{https://openreview.net/forum?id=HJz6tiCqYm}.

\bibitem[Hendrycks et~al.(2018)Hendrycks, Mazeika, Wilson, and
  Gimpel]{hendrycks2018using}
Dan Hendrycks, Mantas Mazeika, Duncan Wilson, and Kevin Gimpel.
\newblock Using trusted data to train deep networks on labels corrupted by
  severe noise.
\newblock \emph{Advances in neural information processing systems}, 31, 2018.

\bibitem[Jun et~al.(2017)Jun, Orabona, Wright, and Willett]{jun2017cbce}
Kwang-Sung Jun, Francesco Orabona, Stephen Wright, and Rebecca Willett.
\newblock {Improved Strongly Adaptive Online Learning using Coin Betting}.
\newblock In Aarti Singh and Jerry Zhu, editors, \emph{Proceedings of the 20th
  International Conference on Artificial Intelligence and Statistics},
  volume~54 of \emph{Proceedings of Machine Learning Research}, pages 943--951.
  PMLR, 20--22 Apr 2017.
\newblock URL \url{https://proceedings.mlr.press/v54/jun17a.html}.

\bibitem[Kath and Ziel(2021)]{kath2021power}
Christopher Kath and Florian Ziel.
\newblock Conformal prediction interval estimation and applications to
  day-ahead and intraday power markets.
\newblock \emph{International Journal of Forecasting}, 37\penalty0
  (2):\penalty0 777--799, 2021.
\newblock ISSN 0169-2070.
\newblock \doi{https://doi.org/10.1016/j.ijforecast.2020.09.006}.
\newblock URL
  \url{https://www.sciencedirect.com/science/article/pii/S0169207020301473}.

\bibitem[Koenker and Bassett~Jr(1978)]{koenker1978regression}
Roger Koenker and Gilbert Bassett~Jr.
\newblock Regression quantiles.
\newblock \emph{Econometrica: journal of the Econometric Society}, pages
  33--50, 1978.

\bibitem[Kwiatkowski et~al.(1992)Kwiatkowski, Phillips, Schmidt, and
  Shin]{kpss1992}
Denis Kwiatkowski, Peter~C.B. Phillips, Peter Schmidt, and Yongcheol Shin.
\newblock Testing the null hypothesis of stationarity against the alternative
  of a unit root: How sure are we that economic time series have a unit root?
\newblock \emph{Journal of Econometrics}, 54\penalty0 (1):\penalty0 159--178,
  1992.
\newblock ISSN 0304-4076.
\newblock \doi{https://doi.org/10.1016/0304-4076(92)90104-Y}.
\newblock URL
  \url{https://www.sciencedirect.com/science/article/pii/030440769290104Y}.

\bibitem[Le and Yang(2015)]{Le2015TinyIV}
Ya~Le and Xuan~S. Yang.
\newblock Tiny imagenet visual recognition challenge.
\newblock 2015.

\bibitem[Lei and Wasserman(2014)]{lei2014distribution}
Jing Lei and Larry Wasserman.
\newblock Distribution-free prediction bands for non-parametric regression.
\newblock \emph{Journal of the Royal Statistical Society: Series B (Statistical
  Methodology)}, 76\penalty0 (1):\penalty0 71--96, 2014.
\newblock \doi{https://doi.org/10.1111/rssb.12021}.
\newblock URL
  \url{https://rss.onlinelibrary.wiley.com/doi/abs/10.1111/rssb.12021}.

\bibitem[Lei et~al.(2013)Lei, Robins, and Wasserman]{lei2013classification}
Jing Lei, James Robins, and Larry Wasserman.
\newblock Distribution-free prediction sets.
\newblock \emph{Journal of the American Statistical Association}, 108\penalty0
  (501):\penalty0 278--287, 2013.
\newblock \doi{10.1080/01621459.2012.751873}.
\newblock URL \url{https://doi.org/10.1080/01621459.2012.751873}.
\newblock PMID: 25237208.

\bibitem[Makridakis et~al.(2018)Makridakis, Spiliotis, and
  Assimakopoulos]{makridakis2018m4}
Spyros Makridakis, Evangelos Spiliotis, and Vassilios Assimakopoulos.
\newblock The m4 competition: Results, findings, conclusion and way forward.
\newblock \emph{International Journal of Forecasting}, 34\penalty0
  (4):\penalty0 802--808, 2018.

\bibitem[Orabona(2019)]{orabona2019modern}
Francesco Orabona.
\newblock A modern introduction to online learning.
\newblock \emph{arXiv preprint arXiv:1912.13213}, 2019.

\bibitem[Orabona and P{\'a}l(2016)]{orabona2016coin}
Francesco Orabona and D{\'a}vid P{\'a}l.
\newblock Coin betting and parameter-free online learning.
\newblock \emph{Advances in Neural Information Processing Systems}, 29, 2016.

\bibitem[Orabona and Pál(2018)]{orabona2018ogd}
Francesco Orabona and Dávid Pál.
\newblock Scale-free online learning.
\newblock \emph{Theoretical Computer Science}, 716:\penalty0 50--69, 2018.
\newblock ISSN 0304-3975.
\newblock \doi{https://doi.org/10.1016/j.tcs.2017.11.021}.
\newblock URL
  \url{https://www.sciencedirect.com/science/article/pii/S0304397517308514}.
\newblock Special Issue on ALT 2015.

\bibitem[Papadopoulos(2008)]{papadopoulos2008inductive}
Harris Papadopoulos.
\newblock Inductive conformal prediction: Theory and application to neural
  networks.
\newblock In Paula Fritzsche, editor, \emph{Tools in Artificial Intelligence},
  chapter~18. IntechOpen, Rijeka, 2008.
\newblock \doi{10.5772/6078}.
\newblock URL \url{https://doi.org/10.5772/6078}.

\bibitem[Park et~al.(2020)Park, Bastani, Matni, and Lee]{Park2020PAC}
Sangdon Park, Osbert Bastani, Nikolai Matni, and Insup Lee.
\newblock Pac confidence sets for deep neural networks via calibrated
  prediction.
\newblock In \emph{International Conference on Learning Representations}, 2020.
\newblock URL \url{https://openreview.net/forum?id=BJxVI04YvB}.

\bibitem[Paszke et~al.(2019)Paszke, Gross, Massa, Lerer, Bradbury, Chanan,
  Killeen, Lin, Gimelshein, Antiga, Desmaison, Kopf, Yang, DeVito, Raison,
  Tejani, Chilamkurthy, Steiner, Fang, Bai, and Chintala]{pytorch}
Adam Paszke, Sam Gross, Francisco Massa, Adam Lerer, James Bradbury, Gregory
  Chanan, Trevor Killeen, Zeming Lin, Natalia Gimelshein, Luca Antiga, Alban
  Desmaison, Andreas Kopf, Edward Yang, Zachary DeVito, Martin Raison, Alykhan
  Tejani, Sasank Chilamkurthy, Benoit Steiner, Lu~Fang, Junjie Bai, and Soumith
  Chintala.
\newblock Pytorch: An imperative style, high-performance deep learning library.
\newblock In \emph{Advances in Neural Information Processing Systems 32}, pages
  8024--8035. Curran Associates, Inc., 2019.

\bibitem[Pearce et~al.(2018)Pearce, Brintrup, Zaki, and Neely]{pearce2018high}
Tim Pearce, Alexandra Brintrup, Mohamed Zaki, and Andy Neely.
\newblock High-quality prediction intervals for deep learning: A
  distribution-free, ensembled approach.
\newblock In Jennifer Dy and Andreas Krause, editors, \emph{Proceedings of the
  35th International Conference on Machine Learning}, volume~80 of
  \emph{Proceedings of Machine Learning Research}, pages 4075--4084. PMLR,
  10--15 Jul 2018.
\newblock URL \url{https://proceedings.mlr.press/v80/pearce18a.html}.

\bibitem[Podkopaev and Ramdas(2021)]{podkopaev2021label_shift}
Aleksandr Podkopaev and Aaditya Ramdas.
\newblock Distribution-free uncertainty quantification for classification under
  label shift.
\newblock In Cassio de~Campos and Marloes~H. Maathuis, editors,
  \emph{Proceedings of the Thirty-Seventh Conference on Uncertainty in
  Artificial Intelligence}, volume 161 of \emph{Proceedings of Machine Learning
  Research}, pages 844--853. PMLR, 27--30 Jul 2021.
\newblock URL \url{https://proceedings.mlr.press/v161/podkopaev21a.html}.

\bibitem[Romano et~al.(2019)Romano, Patterson, and Candes]{romano2019cqr}
Yaniv Romano, Evan Patterson, and Emmanuel Candes.
\newblock Conformalized quantile regression.
\newblock In H.~Wallach, H.~Larochelle, A.~Beygelzimer, F.~d\textquotesingle
  Alch\'{e}-Buc, E.~Fox, and R.~Garnett, editors, \emph{Advances in Neural
  Information Processing Systems}, volume~32. Curran Associates, Inc., 2019.
\newblock URL
  \url{https://proceedings.neurips.cc/paper/2019/file/5103c3584b063c431bd1268e9b5e76fb-Paper.pdf}.

\bibitem[Romano et~al.(2020)Romano, Sesia, and Cand\`{e}s]{romano2020aps}
Yaniv Romano, Matteo Sesia, and Emmanuel~J. Cand\`{e}s.
\newblock Classification with valid and adaptive coverage.
\newblock In \emph{Proceedings of the 34th International Conference on Neural
  Information Processing Systems}, NIPS'20, Red Hook, NY, USA, 2020. Curran
  Associates Inc.
\newblock ISBN 9781713829546.

\bibitem[Shafer and Vovk(2008)]{shafer2008tutorial}
Glenn Shafer and Vladimir Vovk.
\newblock A tutorial on conformal prediction.
\newblock \emph{Journal of Machine Learning Research}, 9\penalty0 (3), 2008.

\bibitem[Sousa et~al.(2022)Sousa, Tomé, and Moreira]{sousa2022general}
Martim Sousa, Ana~Maria Tomé, and José Moreira.
\newblock A general framework for multi-step ahead adaptive conformal
  heteroscedastic time series forecasting, 2022.
\newblock URL \url{https://arxiv.org/abs/2207.14219}.

\bibitem[Stankeviciute et~al.(2021)Stankeviciute, M.~Alaa, and van~der
  Schaar]{stankeviciute2021conformal_ts_exchangeable}
Kamile Stankeviciute, Ahmed M.~Alaa, and Mihaela van~der Schaar.
\newblock Conformal time-series forecasting.
\newblock In M.~Ranzato, A.~Beygelzimer, Y.~Dauphin, P.S. Liang, and J.~Wortman
  Vaughan, editors, \emph{Advances in Neural Information Processing Systems},
  volume~34, pages 6216--6228. Curran Associates, Inc., 2021.
\newblock URL
  \url{https://proceedings.neurips.cc/paper/2021/file/312f1ba2a72318edaaa995a67835fad5-Paper.pdf}.

\bibitem[Steinwart and Christmann(2011)]{steinwart2011pinball}
Ingo Steinwart and Andreas Christmann.
\newblock {Estimating conditional quantiles with the help of the pinball loss}.
\newblock \emph{Bernoulli}, 17\penalty0 (1):\penalty0 211 -- 225, 2011.
\newblock \doi{10.3150/10-BEJ267}.
\newblock URL \url{https://doi.org/10.3150/10-BEJ267}.

\bibitem[Stutz et~al.(2022)Stutz, Dvijotham, Cemgil, and
  Doucet]{stutz2022learning}
David Stutz, Krishnamurthy~Dj Dvijotham, Ali~Taylan Cemgil, and Arnaud Doucet.
\newblock Learning optimal conformal classifiers.
\newblock In \emph{International Conference on Learning Representations}, 2022.
\newblock URL \url{https://openreview.net/forum?id=t8O-4LKFVx}.

\bibitem[Sun and Yu(2022)]{sun2022copula}
Sophia Sun and Rose Yu.
\newblock Copula conformal prediction for multi-step time series forecasting,
  2022.
\newblock URL \url{https://arxiv.org/abs/2212.03281}.

\bibitem[Taylor and Letham(2017)]{taylor2017prophet}
Sean~J. Taylor and Benjamin Letham.
\newblock Forecasting at scale.
\newblock \emph{PeerJ Preprints}, 5\penalty0 (e3190v2), Sept 2017.
\newblock \doi{10.7287/peerj.preprints.3190v2}.

\bibitem[Tibshirani et~al.(2019)Tibshirani, Foygel~Barber, Candes, and
  Ramdas]{tibshirani2019cov_shift}
Ryan~J Tibshirani, Rina Foygel~Barber, Emmanuel Candes, and Aaditya Ramdas.
\newblock Conformal prediction under covariate shift.
\newblock In H.~Wallach, H.~Larochelle, A.~Beygelzimer, F.~d\textquotesingle
  Alch\'{e}-Buc, E.~Fox, and R.~Garnett, editors, \emph{Advances in Neural
  Information Processing Systems}, volume~32. Curran Associates, Inc., 2019.
\newblock URL
  \url{https://proceedings.neurips.cc/paper/2019/file/8fb21ee7a2207526da55a679f0332de2-Paper.pdf}.

\bibitem[Vovk(2012)]{vovk2012conditional}
Vladimir Vovk.
\newblock Conditional validity of inductive conformal predictors.
\newblock In Steven C.~H. Hoi and Wray Buntine, editors, \emph{Proceedings of
  the Asian Conference on Machine Learning}, volume~25 of \emph{Proceedings of
  Machine Learning Research}, pages 475--490, Singapore Management University,
  Singapore, 04--06 Nov 2012. PMLR.
\newblock URL \url{https://proceedings.mlr.press/v25/vovk12.html}.

\bibitem[Vovk et~al.(2005)Vovk, Gammerman, and Shafer]{vovk2005alrw}
Vladimir Vovk, Alex Gammerman, and Glenn Shafer.
\newblock \emph{Algorithmic Learning in a Random World}.
\newblock Springer-Verlag, Berlin, Heidelberg, 2005.
\newblock ISBN 0387001522.

\bibitem[Vovk et~al.(2018)Vovk, Nouretdinov, Manokhin, and
  Gammerman]{vovk2018cross}
Vladimir Vovk, Ilia Nouretdinov, Valery Manokhin, and Alexander Gammerman.
\newblock Cross-conformal predictive distributions.
\newblock In Alex Gammerman, Vladimir Vovk, Zhiyuan Luo, Evgueni Smirnov, and
  Ralf Peeters, editors, \emph{Proceedings of the Seventh Workshop on Conformal
  and Probabilistic Prediction and Applications}, volume~91 of
  \emph{Proceedings of Machine Learning Research}, pages 37--51. PMLR, 11--13
  Jun 2018.
\newblock URL \url{https://proceedings.mlr.press/v91/vovk18a.html}.

\bibitem[Vovk et~al.(1999)Vovk, Gammerman, and Saunders]{vovk1999original}
Volodya Vovk, Alexander Gammerman, and Craig Saunders.
\newblock Machine-learning applications of algorithmic randomness.
\newblock In \emph{Proceedings of the Sixteenth International Conference on
  Machine Learning}, ICML '99, page 444–453, San Francisco, CA, USA, 1999.
  Morgan Kaufmann Publishers Inc.
\newblock ISBN 1558606122.

\bibitem[Wisniewski et~al.(2020)Wisniewski, Lindsay, and
  Lindsay]{wisniewski2020application}
Wojciech Wisniewski, David Lindsay, and Sian Lindsay.
\newblock Application of conformal prediction interval estimations to market
  makers’ net positions.
\newblock In Alexander Gammerman, Vladimir Vovk, Zhiyuan Luo, Evgueni Smirnov,
  and Giovanni Cherubin, editors, \emph{Proceedings of the Ninth Symposium on
  Conformal and Probabilistic Prediction and Applications}, volume 128 of
  \emph{Proceedings of Machine Learning Research}, pages 285--301. PMLR, 09--11
  Sep 2020.
\newblock URL \url{https://proceedings.mlr.press/v128/wisniewski20a.html}.

\bibitem[Xu and Xie(2021)]{xu2021enbpi}
Chen Xu and Yao Xie.
\newblock Conformal prediction interval for dynamic time-series.
\newblock In Marina Meila and Tong Zhang, editors, \emph{Proceedings of the
  38th International Conference on Machine Learning}, volume 139 of
  \emph{Proceedings of Machine Learning Research}, pages 11559--11569. PMLR,
  18--24 Jul 2021.
\newblock URL \url{https://proceedings.mlr.press/v139/xu21h.html}.

\bibitem[Yang and Kuchibhotla(2021)]{yang2021efficient}
Yachong Yang and Arun~Kumar Kuchibhotla.
\newblock Finite-sample efficient conformal prediction, 2021.
\newblock URL \url{https://arxiv.org/abs/2104.13871}.

\bibitem[Yang et~al.(2022)Yang, Kuchibhotla, and Tchetgen]{yang2022robust}
Yachong Yang, Arun~Kumar Kuchibhotla, and Eric~Tchetgen Tchetgen.
\newblock Doubly robust calibration of prediction sets under covariate shift,
  2022.
\newblock URL \url{https://arxiv.org/abs/2203.01761}.

\bibitem[Zaffran et~al.(2022)Zaffran, Feron, Goude, Josse, and
  Dieuleveut]{zaffran2022ts_aci}
Margaux Zaffran, Olivier Feron, Yannig Goude, Julie Josse, and Aymeric
  Dieuleveut.
\newblock Adaptive conformal predictions for time series.
\newblock In Kamalika Chaudhuri, Stefanie Jegelka, Le~Song, Csaba Szepesvari,
  Gang Niu, and Sivan Sabato, editors, \emph{Proceedings of the 39th
  International Conference on Machine Learning}, volume 162 of
  \emph{Proceedings of Machine Learning Research}, pages 25834--25866. PMLR,
  17--23 Jul 2022.
\newblock URL \url{https://proceedings.mlr.press/v162/zaffran22a.html}.

\bibitem[Zhang et~al.(2018)Zhang, Yang, rong jin, and Zhou]{zhang2018dynamic}
Lijun Zhang, Tianbao Yang, rong jin, and Zhi-Hua Zhou.
\newblock Dynamic regret of strongly adaptive methods.
\newblock In Jennifer Dy and Andreas Krause, editors, \emph{Proceedings of the
  35th International Conference on Machine Learning}, volume~80 of
  \emph{Proceedings of Machine Learning Research}, pages 5882--5891. PMLR,
  10--15 Jul 2018.
\newblock URL \url{https://proceedings.mlr.press/v80/zhang18o.html}.

\bibitem[Zinkevich(2003)]{zinkevich2003olo}
Martin Zinkevich.
\newblock Online convex programming and generalized infinitesimal gradient
  ascent.
\newblock In \emph{Proceedings of the Twentieth International Conference on
  International Conference on Machine Learning}, ICML'03, page 928–935. AAAI
  Press, 2003.
\newblock ISBN 1577351894.

\end{thebibliography}
\bibliographystyle{plainnat}

\appendix

\section{Basic Properties of Online Conformal Prediction Algorithms}

\subsection{Properties of \methodBasic{}}
We consider the \methodBasic{} algorithm (Algorithm~\ref{alg:ogd}). We first show that the iterates of \methodBasic{} are bounded within a range slightly larger than the range of the true radii. The proof is similar to~\citet[Lemma 4.1]{gibbs2021aci}.

\begin{lemma}[Bounded iterates for \methodBasic{}]
\label{lem:bounded_iterates_ogd}
Suppose the true radii are bounded: $S_t\in[0,D]$ for all $t\in[T]$. Then Algorithm~\ref{alg:ogd} with any initialization $\hat{s}_1\in[-\eta,D+\eta]$ and learning rate $\eta>0$ admits bounded iterates:
\begin{align*}
\hat{s}_t \in \brac{-\eta, D+\eta}~~~\textrm{for all}~t\in[T].
\end{align*}    
\end{lemma} 
\begin{proof}
Recall by~\eqref{eq:grad_loss} that
\begin{align}
\label{eq:grad_loss_bound}
\grad \ellt(\hat{s}_t) = \alpha - \indic{\hat{s}_t < S_t} = \alpha - \err_t \in \set{-(1-\alpha), \alpha} \subset [-1, 1].
\end{align}
for all $t\in[T]$. Therefore, Algorithm~\ref{alg:ogd} satisfies for any $t\ge 1$ that
\begin{align}
\label{eq:bounded_diff_ogd}
\abs{\hat{s}_{t+1} - \hat{s}_t} = \eta\abs{ \frac{\alpha - \err_t}{\sqrt{\sum_{\tau=1}^t (\alpha - \err_\tau)^2}} } \le \eta.
\end{align}
We prove the lemma by contradiction. Suppose there exists some $t$ such that $\hat{s}_t\notin [-\eta, D+\eta]$. Let $t\ge 2$ be the smallest such time index (abusing notation slightly). Suppose $\hat{s}_t>D+\eta$, then by~\eqref{eq:bounded_diff_ogd} we must have $\hat{s}_{t-1}>D$ but $\hat{s}_{t-1}\le D+\eta$. Note that $\hat{s}_{t-1}>D\ge S_{t-1}$ by our precondition, so that the $(t-1)$-th prediction set must cover and thus $\err_{t-1}=0$. Therefore by the algorithm update~\eqref{eq:ogd_update} we have
\begin{align*}
\hat{s}_t = \hat{s}_{t-1} - \eta \frac{\alpha-\err_{t-1}}{\sqrt{\sum_{\tau=1}^{t-1} (\alpha-\err_{\tau})^2}} < \hat{s}_{t-1} \le D+\eta,
\end{align*}
contradicting with our assumption that $\hat{s}_t>D+\eta$. A similar contradiction can be derived for the other case where $\hat{s}_t<-\eta$. This proves the desired result.
\end{proof}

The following regret bound follows directly by applying the generic regret bound of Scale-Free OGD~\citep[Theorem 2]{orabona2018ogd} to the quantile loss~\eqref{eq:quantile_loss}.

\begin{proposition}[Anytime regret bound for \methodBasic{}]
\label{prop:ogd_regret}
Suppose the true radii are bounded: $S_t\in[0,D]$ for all $t\in[T]$. Then Algorithm~\ref{alg:ogd} with any initialization $\hat{s}_1\in[0,D]$ and learning rate $\eta=D/\sqrt{3}$ achieves the following regret bound for any $t\in[T]$:
\begin{align*}
\Reg(t) \le (\sqrt{3}+1)D \sqrt{\sum_{\tau=1}^{t}\norm{\nabla \elltau(\hat{s}_\tau)}_2^2} \le \cO(D\sqrt{t}).
\end{align*}
\end{proposition}
\begin{proof}
The second inequality follows directly by~\eqref{eq:grad_loss_bound}.

To prove the first inequality (the regret bound), we note that Algorithm~\ref{alg:ogd} is a special case of the Scale-Free Mirror Descent algorithm of~\citet[Section 4]{orabona2018ogd} with convex loss $\ell^{(t)}(\cdot)=\ell_{1-\alpha}(S_t, \cdot)$, and regularizer $R(s)\defeq s^2/(2\eta)$ (in their notation) which is $\lambda=1/\eta$-strongly convex with respect to the $\ell_2$ norm on $\R$. Further, by Lemma~\ref{lem:bounded_iterates_ogd} we have $\hat{s}_t\in[-\eta, D+\eta]$ for all $t\in[T]$. Therefore, applying \citet[Theorem 2]{orabona2018ogd} gives that for any $t\in[T]$,
\begin{align*}
& \quad \sum_{\tau=1}^t \elltau(\hat{s}_\tau) - \inf_{s^\star\in[0, D]}\sum_{\tau=1}^t \elltau(s^\star) \le \paren{\frac{1}{\lambda} + \sup_{\tau\ge 1}B_R(s^\star, \hat{s}_\tau)} \cdot \sqrt{\sum_{\tau=1}^{t}\norm{\nabla \elltau(\hat{s}_\tau)}_2^2} \\
& = \paren{ \eta + \sup_{s^\star\in[0,D], s'\in[-\eta, D+\eta]}\frac{1}{2\eta}(s^\star-s')^2} \cdot \sqrt{\sum_{\tau=1}^{t}\norm{\nabla \elltau(\hat{s}_\tau)}_2^2} \\
& = \paren{\eta + \frac{(D+\eta)^2}{2\eta}} \cdot \sqrt{\sum_{\tau=1}^{t}\norm{\nabla \elltau(\hat{s}_\tau)}_2^2},
\end{align*}
where $B_R(\cdot, \cdot)$ denotes the Bregman divergence associated with $R$.
Choosing $\eta=D/\sqrt{3}$, the leading coefficient is $(\sqrt{3}+1)D$. The desired result follows by noting that
\begin{align*}
\Reg(t) = \sum_{\tau=1}^t \elltau(\hat{s}_\tau) - \inf_{s^\star\in\R}\sum_{\tau=1}^t 
 \elltau(s^\star) = \sum_{\tau=1}^t \elltau(\hat{s}_\tau) - \inf_{s^\star\in[0, D]}\sum_{\tau=1}^t \elltau(s^\star)
\end{align*}
by our assumption that $S_t\in[0,D]$ and basic properties of the quantile losses $\set{\elltau(\cdot)}_{\tau \ge 1}$.
\end{proof}

\subsection{Example of ``Trivial'' Algorithm with Coverage Guarantee}
\label{appendix:trivial_alg}

We consider the following ``trivial'' online conformal prediction algorithm that does not utilize the data at all: Simply predict the maximum radius $D$ for $(1-\alpha)$ proportion of the time steps, then predict the minimum radius $0$ for $\alpha$ proportion of the time steps:
\begin{align}
\label{eq:trivial_alg}
\left\{
\begin{aligned}
& \hat{s}_t \defeq D~~~\textrm{for}~~t\in\set{1, \dots, \floor{(1-\alpha)T}} \eqdef T_{\full}, \\
& \hat{s}_t \defeq 0~~~\textrm{for}~~t\in\set{\floor{(1-\alpha)T} + 1, \dots, T} \eqdef T_{\emp}.
\end{aligned}
\right.
\end{align}
By our assumption that $S_t\in[0,D]$ almost surely and the nested set structure of $\hat{C}_t(X_t, \cdot)$, we have $\err_t=\indics{Y_t\in\hat{C}_t}=\indics{\hat{s}_t \ge S_t}=0$ for all $t\in T_{\full}$, and similarly $\err_t=1$ for all $t\in T_{\emp}$. Therefore, algorithm~\eqref{eq:trivial_alg} directly satisfies
\begin{align}
\Err(T) = \abs{\frac{1}{T}\sum_{t=1}^T \err_t - \alpha} = \abs{\frac{\abs{T_{\emp}}}{T} - \alpha} =  \abs{\frac{ T - \floor{(1-\alpha)T} }{T} - \alpha} \le \frac{1}{T},
\end{align}
i.e. the algorithm achieves approximately $(1-\alpha)$ empirical coverage, with error $O(1/T)$. It is also straightforward to see that, by slightly modifying the definition of $T_{\full},T_{\emp}$ (making the two index sets alternate), we can make the above coverage bound hold in an anytime sense (for $t\in[T]$).

However, it is straightforward to construct examples of data distributions for which the trivial algorithm~\eqref{eq:trivial_alg} suffers linear regret on the quantile loss $\ellt$ defined in~\eqref{eq:quantile_loss}, and such data distributions can be chosen to be fairly simple. For example, suppose all data points admit the same true radius $D/2$, i.e. 
\begin{align*}
    S_t\equiv D/2~~~\textrm{for all}~t\in[T].
\end{align*}
Then for $s^\star=D/2$ we have $\ellt(s^\star)=\ell_{1-\alpha}(D/2, D/2)=0$ for all $t\in[T]$, which achieves total loss $\sum_{t=1}^T\ellt(s^\star)=0$ (the smallest possible, since $\ellt(\cdot)\ge 0$). On the other hand, algorithm~\eqref{eq:trivial_alg} achieves loss
\begin{align*}
\ellt(\hat{s}_t) = \ell_{1-\alpha}(S_t, \hat{s}_t) = \left\{
\begin{aligned}
& \ell_{1-\alpha}(D/2, D) = \alpha(D/2) ~~~\textrm{for}~~t\in T_{\full}, \\
& \ell_{1-\alpha}(D/2, 0) = (1-\alpha)(D/2)~~~\textrm{for}~~t\in T_{\emp}.
\end{aligned}
\right.
\end{align*}
Therefore, we have
\begin{align*}
& \quad \Reg(T) = \sum_{t=1}^T \ellt(\hat{s}_t) - \inf_{s^\star} \ellt(s_\star) = \sum_{t=1}^T \ellt(\hat{s}_t) = \alpha D/2 \cdot \abs{T_{\full}} + (1-\alpha) D/2 \cdot \abs{T_{\emp}} \\
& = \alpha D/2 \cdot \floor{(1-\alpha)T} + (1-\alpha)D/2 \cdot (T - \floor{(1-\alpha)T}) \ge \alpha(1-\alpha)D T = \Omega(T),
\end{align*}
i.e. algorithm~\eqref{eq:trivial_alg} suffers from linear regret. This demonstrates sublinear regret as a sensible criterion for ruling out trivial algorithms like~\eqref{eq:trivial_alg}.

\section{Proofs for Section~\ref{sec:theory}}

\subsection{Proof of Proposition~\ref{prop:saregret}}
\label{appendix:proof_saregret}

The proof follows by plugging in the regret bound for \methodBasic{} (Proposition~\ref{prop:ogd_regret}) into \citet[Theorem 2]{jun2017cbce}. Define $u(t) \defeq \max_n \{2^n : t \equiv 0 \text{ mod } 2^n\}$. Fix any $k\in[T]$ and $\tau\in[T-k+1]$. Their proof starts by splitting the interval $[\tau, \tau + k - 1]$ into consecutive sub-intervals $\bar{J}^{(1)}, \ldots, \bar{J}^{(n)}$, where $\bar{J}^{(i)} = [\tau_i, \max\{\tau + k, \tau_i + u(\tau_i)\} - 1]$ is a prefix of expert $\mathcal{A}_{\tau_i}$'s active interval.

We have for any fixed $s^\star\in\R$ that
\begin{align*}
\mathrm{Regret}_\tau^k(s^\star) &\defeq \sum_{t=\tau}^{\tau + k - 1} \ell^{(t)}(\hat{s}_t) - \sum_{t=\tau}^{\tau + k - 1} \ell^{(t)}(s^\star) \\
&= \sum_{i=1}^{n} \sum_{t \in \bar{J}^{(i)}} \qty( \ell^{(t)}(\hat{s}_t) - \ell^{(t)}(\hat{s}_{t,\tau_i}) ) + \sum_{i=1}^{n} \sum_{t \in \bar{J}^{(i)}} \qty( \ell^{(t)}(\hat{s}_{t,\tau_i}) - \ell^{(t)}(s^\star) ) \\
&\le D \sum_{i=1}^{n} \underbrace{\sqrt{\abs{\bar{J}^{(i)}} (7 \log T + 5)}}_{\text{\citet[Lemma 2]{jun2017cbce}}} + D \sum_{i=1}^{n} \underbrace{\sqrt{\abs{\bar{J}^{(i)}} (1 + \sqrt{3})}}_{\text{Proposition~\ref{prop:ogd_regret}}} \\
&\le D \qty(\sqrt{7 \log T + 5} + \sqrt{1 + \sqrt{3}}) \underbrace{\sum_{j=0}^{\infty} \sqrt{k2^{-j}}}_{\text{\citet[Lemma 3]{jun2017cbce}}} \\
&= \frac{D \sqrt{2}}{\sqrt{2} - 1}  \sqrt{k} \qty(\sqrt{7 \log T + 5} + \sqrt{1 + \sqrt{3}}) \le 15 D \sqrt{k (\log T + 1)}
\end{align*}
Taking supremum over all $s^\star\in\R$ and all intervals $[\tau, \tau + k - 1]\subset [T]$, we obtain the desired bound on $\SAReg(T, k)$.
\qed

\subsection{Dynamic Regret for \method{}}
\label{appendix:proof_dynamic_regret}
\begin{proposition}[Dynamic regret bound for \method{}]
\label{prop:cbce_dynamic_regret}
Algorithm~\ref{alg:cbce} achieves the following dynamic regret bound: For any interval $[\tau, \tau+k-1]\subset [T]$ of length $k\in[T]$, we have
\begin{align}
\label{eqn:cbce_dynamic_regret}
 \sum_{t=\tau}^{\tau+k-1} \ell^{(t)}(\hat{s}_t) - \min_{s^\star_{\tau:\tau+k-1}} \sum_{t=\tau}^{\tau+k-1} \ell^{(t)}(s_t^\star) = 
\sum_{t=\tau}^{\tau+k-1} \brac{ \ell^{(t)}(\hat{s}_t) - \ell^{(t)}(S_t) } \le \tO\paren{ D\brac{ V_{[\tau, \tau+k-1]}^{1/3}k^{2/3} + \sqrt{k}} },
\end{align}
where
\begin{align*}
V_{[\tau, \tau+k-1]} \defeq \sum_{t=\tau+1}^{\tau+k-1} \abs{S_t - S_{t-1}}
\end{align*}
is the \emph{path length} of the true radii within $[\tau, \tau+k-1]$.
\end{proposition}

We remark that the dynamic regret\footnote{More precisely, the intermediate result with $V_{[\tau, \tau+k-1]}$ replaced by the standard total variation of losses $\wt{V}_{[\tau, \tau+k-1]}$ in~\eqref{eq:total_variation}.} $\tO(V_{[\tau, \tau+k-1]}^{1/3}k^{2/3} + \sqrt{k})$ is the minimax optimal dynamic regret~\citep{besbes2015non} for general online convex optimization problems. 

\paragraph{Comparison of dynamic regret with FACI}
Dividing~\eqref{eqn:cbce_dynamic_regret} by $k$, we obtain the following average dynamic regret bound for \method{} on $[\tau, \tau+k-1]$:
\begin{align*}
\tO\paren{ D\brac{ (V_{[\tau, \tau+k-1]}/k)^{1/3} + 1/\sqrt{k}} },
\end{align*}
simultaneously for all lengths $k$ and $\tau\in[T-k+1]$. 

In comparison, the FACI algorithm (adapted to our setting) with learning rate $\eta$ acheives average dynamic regret bound~\citep[Theorem 3.2]{gibbs2022faci}
\begin{align*}
\tO\paren{ D\brac{ (V_{[\tau, \tau+k-1]}/k)^{1/2} + \eta/D + D/(\eta k) } }.
\end{align*}
When the path length $V_{[\tau, \tau+k-1]}=o(k)$, FACI achieves a better dependence on the average path length $(V_{[\tau, \tau+k-1]}/k=o(1)$, yet a worse dependence on $k$ itself due to the inability to choose the optimal $\eta$ simultaneously for all $k$, similar as the comparison of their SARegret bounds (Section~\ref{sec:saregret}).

\begin{proof-of-proposition}[\ref{prop:cbce_dynamic_regret}]
We apply the dynamic regret bound of~\citet[Corollary 5]{zhang2018dynamic} for the \method{} algorithm on the interval $[\tau, \tau+k-1]$, and note that our iterates $\hat{s}_t\in[-\eta, D+\eta] \subset [-D, 2D]$ by Lemma~\ref{lem:bounded_iterates_ogd} and our choice $\eta=D/\sqrt{3}$ in Algorithm~\ref{alg:cbce}. Therefore we obtain
\begin{align*}
\sum_{t=\tau}^{\tau+k-1} \ell^{(t)}(\hat{s}_t) - \min_{s^\star_{\tau:\tau+k-1}}  \sum_{t=\tau}^{\tau+k-1} \ell^{(t)}(s_t^\star) \le \tO\paren{ D\brac{\wt{V}_{[\tau, \tau+k-1]}^{1/3}k^{2/3} + \sqrt{k} } },
\end{align*}
where
\begin{align}
\label{eq:total_variation}
\wt{V}_{[\tau, \tau+k-1]} = \sum_{t=\tau+1}^{\tau+k-1} \sup_{s'\in[0,D]} \abs{\ellt(s') - \elltmone(s')} \stackrel{(i)}{\le} \sum_{t=\tau+1}^{\tau+k-1} |S_t - S_{t-1}| = V_{[\tau, \tau+k-1]},
\end{align}
where (i) follows by the fact that $\abs{\ellt(s')-\elltmone(s')}=\abs{\ell_{1-\alpha}(s',S_t) - \ell_{1-\alpha}(s',S_{t-1})} \le |S_t - S_{t-1}|$ by the $1$-Lipschitzness of the quantile loss~\eqref{eq:quantile_loss} with respect to the second argument. This proves the desired result.
\end{proof-of-proposition}

\subsection{Proof of Theorem~\ref{thm:ogd_coverage}}
\label{appendix:proof_coverage_ogd}
We first note that, by~\eqref{eq:grad_loss} and~\eqref{eq:ogd_update}, Algorithm~\ref{alg:ogd} simplifies to the update
\begin{align}
\label{eq:ogd_update_closedform}
\hat{s}_{t+1} = \hat{s}_t + \eta \frac{\err_t - \alpha}{\sqrt{\sum_{s=1}^t (\err_s - \alpha)^2}} = \hat{s}_1 
 + \eta \sum_{s=1}^{t} \frac{\err_s - \alpha}{\sqrt{\sum_{i=1}^{s} {(\err_i - \alpha)^2}}}.
\end{align}
Note that we have $\hat{s}_{t+1}\in[-\eta, D+\eta]$ for all $t\ge 0$ (Lemma~\ref{lem:bounded_iterates_ogd}), which implies that
\begin{align*}
\abs{\sum_{t=t_0+1}^{t_f} \frac{\err_t-\alpha}{\sqrt{\sum_{s=1}^t (\err_s-\alpha)^2}}} = \frac{1}{\eta} \abs{\hat{s}_{t_f+1} - \hat{s}_{t_0+1}} \le \frac{D+2\eta}{\eta}~~~\textrm{for any}~0\le t_0<t_f.
\end{align*}
Note that $|\err_t-\alpha|\in[\alpha, 1]$ for all $t$. Therefore, we can invoke Lemma~\ref{lem:sqrt_sum} below with $a_t=\err_t-\alpha$ and $M=(D+2\eta)/\eta$ to obtain that for any $T\ge 1$,

\begin{align*}
    \abs{\frac{1}{T} \sum_{t=1}^{T} \err_t - \alpha} \le 2\paren{\frac{D+3\eta}{\eta} + \alpha^{-2} \log T} T^{-1/4}  \le \mathcal{O}(\alpha^{-2} T^{-1/4} \log T),
\end{align*}
where the later bound holds for any $\eta=\Theta(D)$. This proves Theorem~\ref{thm:ogd_coverage}.
\qed

\begin{lemma}
\label{lem:sqrt_sum}
Suppose the sequence $\{a_t\}_{t\in[T]}\in\R$ satisfies $\alpha\le |a_t|\le 1$ for some $\alpha>0$, and
\begin{align*}
\abs{ \sum_{t=t_0+1}^{t_f} \frac{a_t}{\sqrt{\sum_{s=1}^t a_s^2}} } \le M~~~\textrm{for any}~~0\le t_0<t_f \le T.
\end{align*}
Then we have 
\begin{align*}
\abs{ \frac{1}{T}\sum_{t=1}^{T} a_t } \le 2 \paren{M + 1 + \alpha^{-2} \log T} T^{-1/4}.
\end{align*}
\end{lemma}
\begin{proof}
The proof builds on a grouping argument. 
Define integers
\begin{align*}
L = \ceil{T^\beta},~~~K = \ceil{T/L} \le T^{1-\beta} + 1,
\end{align*}
where $\beta\in(0,1)$ is a parameter to be chosen. For any $k\in[K]$, define the $k$-th group to be 
\begin{align}
\label{eq:ogd_grouping}
G_k = \set{t_{k-1}+1,\dots,t_k} \defeq \set{(k-1)L+1,\dots,\min\set{kL,T}},
\end{align}
so that we have $\bigcup_{k=1}^K G_k=[T]$, $|G_k|=L$ for all $k\in[K-1]$, and $|G_K|\le L$.

Next, for any fixed $k\ge 2$, define sums
\begin{align*}
S_k \defeq \sum_{t\in G_k} \frac{a_t}{\sqrt{\sum_{s=1}^t a_s^2}}, ~~~\wt{S}_k \defeq \sum_{t\in G_k} \frac{a_t}{\sqrt{\sum_{s=1}^{t_{k-1}} a_s^2}}.
\end{align*}
By our precondition, we have $\abs{S_k}\le M$ for all $k\in[K]$. Further, we have
\begin{align*}
& \quad \abs{S_k - \wt{S}_k} \le \sum_{t\in G_k} |a_t| \cdot \paren{ \frac{1}{\sqrt{\sum_{s=1}^{t_{k-1}} a_s^2}} - \frac{1}{\sqrt{\sum_{s=1}^t a_s^2}} } \le |G_k| \cdot \paren{ \frac{1}{\sqrt{\sum_{s=1}^{t_{k-1}} a_s^2}} - \frac{1}{\sqrt{\sum_{s=1}^{t_k} a_s^2}} } \\
& \stackrel{(i)}{\le} L \cdot \frac{\sum_{s=t_{k-1}+1}^{t_k} a_s^2}{2\paren{\sum_{s=1}^{t_{k-1}} a_s^2}^{3/2}} \stackrel{(ii)}{\le} L \cdot \frac{L}{2 (\alpha^2(k-1)L) \cdot \paren{\sum_{s=1}^{t_{k-1}} a_s^2}^{1/2}  } = \frac{L}{2\alpha^2(k-1) \cdot \sqrt{\sum_{s=1}^{t_{k-1}} a_s^2} },
\end{align*}
where (i) uses the inequality $\frac{1}{\sqrt{x}}-\frac{1}{\sqrt{x+y}}\le \frac{y}{2x^{3/2}}$ for $x,y\ge 0$, and (ii) uses the bounds $\sum_{s=t_{k-1}+1}^{t_k} a_s^2 \le (t_k - t_{k-1}) \le L$ and $\sum_{s=1}^{t_{k-1}} a_s^2 \ge \alpha^2t_{k-1}= \alpha^2 (k-1)L$.
By the triangle inequality, this implies that
\begin{align*}
\abs{\wt{S}_k} \le \abs{S_k} + \abs{\wt{S}_k - S_k} \le M + \frac{L}{2\alpha^2(k-1) \cdot \sqrt{\sum_{s=1}^{t_{k-1}} a_s^2} },
\end{align*}
and thus for any $k\ge 2$ that
\begin{align*}
& \quad \abs{\sum_{t\in G_k} a_t} = \underbrace{\abs{\sum_{t\in G_k} \frac{a_t}{\sqrt{\sum_{s=1}^{t_{k-1}} a_s^2}}}}_{\abs{\wt{S}_k}} \cdot \sqrt{\sum_{s=1}^{t_{k-1}} a_s^2} \le \paren{ M + \frac{L}{2\alpha^2(k-1) \cdot \sqrt{\sum_{s=1}^{t_{k-1}} a_s^2} } } \cdot \sqrt{\sum_{s=1}^{t_{k-1}} a_s^2} \\
& \le M\sqrt{\sum_{s=1}^{t_{k-1}} a_s^2} + \frac{L}{2\alpha^2(k-1)} \le M\sqrt{(k-1)L} + \frac{L}{2\alpha^2(k-1)}.
\end{align*}
For $k=1$, we have trivially $\abs{\sum_{t\in G_1} a_t} \le |G_1| \le L$. Summing the bounds over $k\in[K]$ yields~\yub{check constants}
\begin{align*}
& \quad \abs{\sum_{t=1}^{T} a_t} \le L + \sum_{k=1}^K \abs{\sum_{t\in G_k} a_t} \le L + M\sqrt{L}\cdot \sum_{k=2}^K \sqrt{k-1} + \frac{L}{2\alpha^2}\sum_{k=2}^K \frac{1}{k-1} \\
& \le L + \frac{2}{3} M\sqrt{L} K^{3/2} + \frac{L}{2\alpha^2} \log_2 K \\
& \le \ceil{T^\beta} + \frac{2}{3} M\sqrt{\ceil{T^{\beta}}} \cdot T^{3(1-\beta)/2} + \frac{1}{2\alpha^2} \ceil{T^\beta} \log_2(T^{1-\beta}) \\
& \le 2 T^\beta + 2 MT^{3/2-\beta} + \frac{2}{\alpha^2} T^\beta\log T.
\end{align*}
Choosing $\beta=3/4$, we obtain
\begin{align*}
\abs{\sum_{t=t_0+1}^{t_f} a_t} \le 2 \paren{M + 1 + \log T/\alpha^2} T^{3/4}.
\end{align*}
Dividing by $T$ on both sides yields the desired result.
\end{proof}

\subsection{Coverage of \method{}}

\begin{theorem}[Coverage bound for \method{}]
\label{thm:cbce_coverage_nonexp_full}
Consider a randomized version of Algorithm~\ref{alg:cbce} where Line~\ref{line:cbce_aggregation} is changed to sampling an expert $i\sim p_{t,\cdot}\in\Delta([t])$ and outputting radius $\hat{s}_{t,i}$. Consider the corresponding expected miscoverage error
\begin{align}
\label{eq:expected-err}
\wt{\err}_t \defeq \sum_{i=1}^t p_{t,i} \underbrace{\indic{\hat{s}_{t,i}<S_t}}_{\defeq \err_{t, i}}.
\end{align}
Then we have for any $T\ge 1$ that
\begin{align*}
\abs{ \frac{1}{T}\sum_{t=1}^{T} \wt{\err}_t - \alpha } \le \cO\Bigg( \inf_{\beta\in(1/2, 1)}\Bigg\{ T^{1/2-\beta} + T^{\beta-1} \times \underbrace{\bigg( 1 + \sum_{j=2}^{\ceil{T^{1-\beta}}} \max_{t\in G_j} \sum_{i=1}^t \abs{p_{t,i} - p_{t_{j-1}, i}\frac{G^i_{i:t_{j-1}}}{G^i_{i:t}}} \bigg)}_{S_\beta( \{p_t\}_{t\in[T]}, \{\sum_{\tau=i}^{t}\ltwos{\grad \elltau(\hat{s}_{\tau,i})}^2\}_{i\le t}\} ) \eqdef S_\beta(T) } \Bigg\} \Bigg)
\end{align*}
(understanding $p_{t_{j-1},i}\defeq 0$ for any $i>t_{j-1}$), where for each $\beta\in(1/2,1)$, $\set{G_j}_{j=1}^{\ceil{T^{1-\beta}}}$ with $|G_j|\le \ceil{T^\beta}$, $G_j=\set{t_{j-1}+1,\dots, \min\set{t_j, T}}$ is the even grouping of $[T]$ as in~\eqref{eq:ogd_grouping}, and
\begin{align*}
G^{i}_{i:t} \defeq \sqrt{ \sum_{\tau=i}^{t} \ltwo{\grad\elltau(\hat{s}_{\tau,i})}^2 } = \sqrt{ \sum_{\tau=i}^{t} (\err_{\tau, i} - \alpha)^2 }
\end{align*}
is the cumulative squared gradients received by expert $\cA_i$ for any $t>i$ (understanding experts as running until time $T$ even after they become inactive).
\end{theorem}
\begin{proof}
Fix any $i\in[T]$. As Algorithm~\ref{alg:cbce} chooses each expert $\cA_i$ to be \methodBasic{} (Algorithm~\ref{alg:ogd}), we have by~\eqref{eq:ogd_update} that for all $t\ge i$,
\begin{align}
\label{eq:s_diff}
\hat{s}_{t+1, i} - \hat{s}_{t,i} = \frac{\eta}{G^i_{i:t}} \cdot (\err_{t,i} - \alpha).
\end{align}

Now fix any $\beta\in(1/2, 1)$. For any group $2\le j\le \ceil{T^{1-\beta}}$ and $t\in G_j$, plugging the above into definition~\eqref{eq:expected-err} gives that
\begin{align*}
& \quad \wt{\err}_t - \alpha = \sum_{i=1}^t p_{t,i}(\err_{t,i} - \alpha) = \frac{1}{\eta} \sum_{i=1}^t p_{t,i} G^i_{i:t}(\hat{s}_{t+1,i} - \hat{s}_{t,i}) \\
& = \frac{1}{\eta} \sum_{i=1}^{t_{j-1}} p_{t_{j-1},i} G^i_{i:t_{j-1}}(\hat{s}_{t+1,i} - \hat{s}_{t,i}) + \frac{1}{\eta} \sum_{i=1}^t \paren{ p_{t,i}G^i_{i:t} - p_{t_{j-1},i} G^i_{i:t_{j-1}}} (\hat{s}_{t+1,i} - \hat{s}_{t,i}) \\
& = \frac{1}{\eta} \sum_{i=1}^{t_{j-1}} p_{t_{j-1},i} G^i_{i:t_{j-1}}(\hat{s}_{t+1,i} - \hat{s}_{t,i}) + \frac{1}{\eta} \sum_{i=1}^t \paren{ p_{t,i} - p_{t_{j-1},i} \frac{G^i_{i:t_{j-1}}}{G^i_{i:t}} } \cdot G^i_{i:t}(\hat{s}_{t+1,i} - \hat{s}_{t,i}).
\end{align*}
Summing this over $t\in G_j$ and noting that the coefficients $p_{t_{j-1},i} G^i_{i:t_{j-1}}$ in the first sum does not depend on $t$, we get
\begin{align*}
& \quad \abs{\sum_{t\in G_j} (\wt{\err}_t - \alpha)} \\
& \le \abs{ \frac{1}{\eta} \sum_{i=1}^{t_{j-1}} p_{t_{j-1},i} G^i_{i:t_{j-1}}(\hat{s}_{t_j+1,i} - \hat{s}_{t_{j-1}+1,i})} + |G_j|\cdot \max_{t\in G_j} \abs{ \frac{1}{\eta} \sum_{i=1}^t \paren{ p_{t,i} - p_{t_{j-1},i} \frac{G^i_{i:t_{j-1}}}{G^i_{i:t}} } \cdot G^i_{i:t}(\hat{s}_{t+1,i} - \hat{s}_{t,i}) } \\
& \le \frac{1}{\eta}\max_{i\in [t_{j-1}]} G^i_{i:t_{j-1}} \abs{\hat{s}_{t_j+1,i} - \hat{s}_{t_{j-1}+1, i}} + |G_j| \cdot \max_{t\in G_j} \sum_{i=1}^t \abs{ \frac{1}{\eta}  \paren{ p_{t,i} - p_{t_{j-1},i} \frac{G^i_{i:t_{j-1}}}{G^i_{i:t}} } \cdot G^i_{i:t}(\hat{s}_{t+1,i} - \hat{s}_{t,i}) } \\
& \stackrel{(i)}{\le} \frac{D+2\eta}{\eta} \sqrt{T} + |G_j| \cdot \max_{t\in G_j} \sum_{i=1}^t \abs{ p_{t,i} - p_{t_{j-1},i} \frac{G^i_{i:t_{j-1}}}{G^i_{i:t}} } \\
& \stackrel{(ii)}{\le} C\sqrt{T} + |G_j| \cdot \max_{t\in G_j} \sum_{i=1}^t \abs{ p_{t,i} - p_{t_{j-1},i} \frac{G^i_{i:t_{j-1}}}{G^i_{i:t}} }
\end{align*}
Above, (i) used $G^i_{i:t_{j-1}}\le \sqrt{t_{j-1}-i+1}\le \sqrt{T}$ by the definition of $G^i_{i:t_{j-1}}$, the bound $\abs{\hat{s}_{t_j+1,i} - \hat{s}_{t_{j-1}+1, i}}\le (D+2\eta)$ which follows by the fact that each expert is initialized within $[-\eta, D+\eta]$ and applying Lemma~\ref{lem:bounded_iterates_ogd}, and the bound $\abs{G^i_{i:t}(\hat{s}_{t+1,i} - \hat{s}_{t,i})} \le \eta$ by~\eqref{eq:s_diff}; (ii) used the fact that $\eta=D/\sqrt{3}$ in Algorithm~\ref{alg:cbce}, so that $(D+2\eta)/\eta=2+\sqrt{3}\eqdef C$ is an absolute constant. Also, note that for group $j=1$, we directly have
\begin{align*}
\abs{\sum_{t\in G_1} (\wt{\err}_t - \alpha)} \le |G_1|.
\end{align*}
Summing all the above bounds over $j\in\brac{ \ceil{T^{1-\beta}} }$ gives
\begin{align*}
& \quad \abs{\sum_{t=1}^T (\wt{\err}_t - \alpha)} \le \sum_{j=1}^{\ceil{T^{1-\beta}}} \abs{\sum_{t\in G_j} (\wt{\err}_t - \alpha)} \\
& \le \cO\Bigg( T^{3/2-\beta} + |G_1| + \sum_{j=2}^{\ceil{T^{1-\beta}}} |G_j| \times \max_{t\in G_j} \sum_{i=1}^t \abs{ p_{t,i} - p_{t_{j-1},i} \frac{G^i_{i:t_{j-1}}}{G^i_{i:t}} } \Bigg) \\
& \le \cO\Bigg( T^{3/2-\beta} + T^\beta \Bigg( 1 + \sum_{j=2}^{\ceil{T^{1-\beta}}} \max_{t\in G_j} \sum_{i=1}^t \abs{ p_{t,i} - p_{t_{j-1},i} \frac{G^i_{i:t_{j-1}}}{G^i_{i:t}}} \Bigg) \Bigg)
\end{align*}
Dividing both sides by $T$ proves the desired bound for this fixed $\beta$. Further taking supremum over $\beta\in(1/2, 1)$ gives the desired result.

\end{proof} 

\subsubsection{Discussions \& \methodBasic{} as a Special Case}
\label{appendix:subsume_discussion}

We first note that, the proof of Theorem~\ref{thm:cbce_coverage_nonexp_full} does not rely on the specific structure of either the expert weights $\set{p_{t,i}}_{i<t}$ or the active intervals. Therefore, the result of Theorem~\ref{thm:cbce_coverage_nonexp_full} holds generically for any other aggregation scheme over experts with arbitrary active intervals, in addition to that specified in Algorithm~\ref{alg:cbce}.

In particular, by setting $p_{t,1}=1$ and $p_{t,i}=0$ for $i\ge 2$, and defining the first expert $\cA_1$ to be active over $[T]$, Algorithm~\ref{alg:cbce} (either with or without the randomization, since there is only one active expert) recovers Algorithm~\ref{alg:ogd}. In this case, we show that $S_\beta(T)\le \tO(\alpha^{-2})$ for any $\beta\in(1/2,1)$, so that Theorem~\ref{thm:cbce_coverage_nonexp_full} (and its informal version in Theorem~\ref{thm:cbce_coverage_nonexp}) indeed subsumes Theorem~\ref{thm:ogd_coverage} as a special case by choosing $\beta=3/4$, as claimed in Section~\ref{sec:coverage}. 

We have
\begin{align}
\label{eq:sbeta_bound}
& \quad S_\beta(T) = 1 + \sum_{j=2}^{\ceil{T^{1-\beta}}} \max_{t\in G_j} \sum_{i=1}^t \abs{ p_{t,i} - p_{t_{j-1},i} \frac{G^i_{i:t_{j-1}}}{G^i_{i:t}}} \stackrel{(i)}{=} 1 + \sum_{j=2}^{\ceil{T^{1-\beta}}} \max_{t\in G_j} \abs{ 1 - \frac{G^1_{1:t_{j-1}}}{G^1_{1:t}}},
\end{align}
where (i) used the fact that $p_{t,1}=1$ and $p_{t,i}=0$ for $i\ge 2$. For any $t\in G_j$, we have
\begin{align*}
\abs{ 1 - \frac{G^1_{1:t_{j-1}}}{G^1_{1:t}}} = 1 - \frac{ \sqrt{\sum_{s=1}^{t_{j-1}} (\err_s - \alpha)^2} }{ \sqrt{\sum_{s=1}^{t} (\err_s - \alpha)^2} } \stackrel{(i)}{\le} \frac{ \sum_{s=t_{j-1}+1}^{t} (\err_s - \alpha)^2 }{ 2\sum_{s=1}^{t_{j-1}} (\err_s - \alpha)^2 } \stackrel{(ii)}{\le} \frac{t-t_{j-1}}{2\alpha^2 t_{j-1}} \stackrel{(iii)}{\le} \frac{\ceil{T^\beta}}{2\alpha^2 \cdot (j-1)\ceil{T^\beta}} = \frac{1}{2\alpha^2(j-1)},
\end{align*}
where (i) follows from the inequality 
$1-\frac{\sqrt{x}}{\sqrt{x+y}} = \frac{\sqrt{x+y}-\sqrt{x}}{\sqrt{x+y}}\le \frac{\sqrt{x+y}-\sqrt{x}}{\sqrt{x}} = \sqrt{1+\frac{y}{x}}-1\le \frac{y}{2x}$ for any $x,y\ge 0$; (ii) follows by the bound $|\err_s-\alpha|\in[\alpha, 1]$ for any $s$; (iii) follows by definition of the grouping~\eqref{eq:ogd_grouping}. Plugging the above bound into~\eqref{eq:sbeta_bound}, we obtain
\begin{align*}
S_\beta(T) \le 1 + \sum_{j=2}^{\ceil{T^{1-\beta}}} \frac{1}{2\alpha^2(j-1)} \le \cO\paren{ \alpha^{-2}\log T } = \tO(\alpha^{-2}),
\end{align*}
proving the claim.

\section{Distribution-Aware Coverage Guarantees for \method{}}
\label{appendix:proof_coverage_cbce}

In this section, we show that under mild density lower bound assumptions on the true radii, a probabilistic variant of the coverage error of \method{} (Algorithm~\ref{alg:cbce}) is bounded by $\tO(k^{-1/(2q)}) + \tO(({\rm Var}_k/k)^{1/q})$ for every interval of length $k$, where $q \ge 2$ is a parameter of the density lower bound assumption, and $\mathrm{Var}_k$ measures a certain variance (over intervals of length $k$) in the $1-\alpha$ conditional quantiles of the true radii. The proof builds on the strongly adaptive regret guarantee (in the quantile loss) for \method{} (Proposition~\ref{thm:cbce}), and bounding parameter estimation errors by excess quantile losses using a \emph{self-calibration inequality} type argument~\citep{steinwart2011pinball}.

\paragraph{Setting}
We consider the online conformal prediction setting described in Section~\ref{sec:related_conformal_ts}. For any $t\ge 1$, let $\mathcal{F}_t \defeq \sigma(\{X_i, \hat{s}_i, S_i\}_{i \in [t-1]}, X_t)$ be the $\sigma$-algebra by all observed data $\{(X_i,\hat{s}_i,S_i)\}_{i\le t-1}$ as well as $X_t$. Note that by definition of the online conformal prediction setting, the predicted radius $\hat{s}_t$ can only depend on information within $\cF_t$ as well as (possibly) external randomness. Consequently, we have $S_t \perp\!\!\!\perp \hat{s}_t \mid \mathcal{F}_{t}$, i.e.\ $S_t$ and $\hat{s}_t$ are conditionally independent given $\mathcal{F}_{t}$. 

We now state our assumptions on the distributions of the true radii.

\begin{assumption}[Density upper bounds]
\label{ass:pdf_ub}
For all $t \in [T]$, there exists a constant $L > 0$ such that $S_t \mid \mathcal{F}_{t}$ is a continuous random variable that is bounded within $[0, D]$ and has a density $f_t: [0, D] \to \R_{\ge 0}$ with $f_t(s) \le L / D$ for all $s \in [0, D]$.
\end{assumption}

\begin{assumption}[Density lower bounds]
\label{ass:pdf_lb}
For all $t \in [T]$, $S_t \mid \mathcal{F}_{t}$ is a continuous random variable that is bounded within $[0, D]$ and has a density $f_t: [0, D] \to \R_{\ge 0}$. With probability one, there exist constants $b > 0, q \ge 2, \Delta_t > 0$ such that
\begin{align}
\label{eqn:pdf_lower_fi}
    f_t(s) \ge \frac{2b}{D} \abs{\frac{2(s - s^\star_t)}{D}}^{q - 2}
\end{align}
for all $s \in [s^\star_t - \Delta_t, s^\star_t + \Delta_t]$, where
\begin{align}
  s^\star_t \defeq Q_{1-\alpha}(S_t \mid \mathcal{F}_{t})
\end{align}
is the $1 - \alpha$ conditional quantile of $S_t$.
\end{assumption}

As examples for Assumption~\ref{ass:pdf_lb}, the case where $q=2$ corresponds to a constant lower bound on the conditional density $f_t(\cdot)$ \emph{locally} around $s_t^\star$, which holds e.g.\ if each $f_t(\cdot)$ itself has a constant lower bound over $[0,D]$ (this is the assumption made by~\citet{gibbs2022faci}). A larger $q$ makes the density lower bound~\eqref{eqn:pdf_lower_fi} easier to satisfy and thus specifies a more relaxed assumption. We also note that $s_t^\star$ is itself a random variable which is measurable on $\cF_t$.

For any interval $I = [\tau, \tau + k - 1] \subseteq [T]$, our coverage result depends on a certain variance between $s_{\tau}^\star, \ldots, s_{\tau+k-1}^\star$. Concretely, define the interval quantile variation
\begin{align}
\label{eqn:interval_variation}
\mathrm{Var}_I \defeq \sum_{t = \tau}^{\tau+k-1} \E\qty[\bigg(\frac{s_t^\star}{D} - \frac{1}{k D} \sum_{i=\tau}^{\tau+k-1} \E[s_i^\star \mid \mathcal{F_{\tau}}]\bigg)^2].
\end{align} 
Then, the expected absolute difference between SAOCP's predictions $\hat{s}_\tau, \ldots, \hat{s}_{\tau+k-1}$ and the true conditional quantiles $s_\tau^\star, \ldots, s_{\tau+k-1}^\star$ is $\tO(k^{-1/(2q)}) + \tO(k^{-1/q} \mathrm{Var}_I^{1/q})$. Due to the Lipschitzness of the CDFs (by Assumption~\ref{ass:pdf_ub}), the coverage error has a similar form. So \method{} achieves better coverage when the interval quantile variation is lower, and it achieves approximately valid coverage as long as $\mathrm{Var}_I \le o(\abs{I})$. More formally, we have:

\begin{theorem}
\label{thm:cbce_coverage}
Let Assumptions~\ref{ass:pdf_ub} \&~\ref{ass:pdf_lb} hold. Fix any interval $I = [\tau, \tau + k - 1] \subseteq [T]$. Then, letting $\Delta_I = \min\{\Delta_t : t \in I\}$, Algorithm~\ref{alg:cbce} achieves quantile estimation error
\begin{align*}
\frac{1}{\abs{I}} \sum_{t \in I} \E\qty[\abs{\frac{\hat{s}_t - s_t^\star}{D}}] 
\le \mathcal{O}\qty(\frac{1}{b^{1/q}}\frac{D}{\Delta_I} \qty(\frac{\log T}{\abs{I}})^{1/(2q)}) + \mathcal{O}\qty(\frac{L^{1/q}}{b^{1/q}} \frac{D}{\Delta_I} \qty(\frac{\mathrm{Var}_I}{\abs{I}})^{1/q})
\end{align*}
and interval miscoverage error
\begin{align*}
    \frac{1}{\abs{I}} \sum_{t \in I} \abs{\P[Y_t \in \hat{C}_t(X_t)] - (1 - \alpha)} \le \mathcal{O}\qty(\frac{L}{b^{1/q}} \frac{D}{\Delta_I} \qty(\frac{\log T}{\abs{I}})^{1/(2q)}) + \mathcal{O}\qty(\frac{L^{1+1/q}}{b^{1/q}} \frac{D}{\Delta_I} \qty(\frac{\mathrm{Var}_I}{\abs{I}})^{1/q} ).
\end{align*}
\end{theorem}
In Theorem~\ref{thm:cbce_coverage}, $q$ is a parameter quantifying the difficulty of lower bounding the distribution of $S_t \mid \mathcal{F}_t$ away from its $1 - \alpha$ conditional quantile $s^\star_t$. If $q$ is higher, then closeness to $s^\star_t$ is less correlated with the expected regret on the quantile loss~\eqref{eq:quantile_loss}. Meanwhile, the term $D/\Delta_I$ grows larger as $\alpha$ grows smaller. The inclusion of this term mirrors the inclusion of $\alpha^{-2}$ in Theorem~\ref{thm:ogd_coverage}, and it indicates that more extreme quantiles are harder to learn.

\begin{proof-of-theorem}[\ref{thm:cbce_coverage}]
The key ingredient is the technical Lemma~\ref{lem:regret_distance}, which uses the expected dynamic regret of a sequence $\hat{s}_\tau, \ldots, \hat{s}_{\tau + k - 1}$ to upper bound the expected distance between that sequence and the true $1 - \alpha$ conditional quantiles of $S_\tau, \ldots, S_{\tau+k-1}$. We decompose the expected dynamic regret into the interval regret of Algorithm~\ref{alg:cbce} (which we can upper bound by Proposition~\ref{prop:saregret}) and a term which we can use $\mathrm{Var}_I$ to upper bound. The desired coverage bound follows by the Lipschitzness of the CDFs of the $S_t$'s (implied by the density upper bound Assumption~\ref{ass:pdf_ub}).  In this proof, we use $\E_{\mathcal{F}_t}[X]$ as short-hand for the conditional expectation $\E[X \mid \mathcal{F}_t]$.
\begin{lemma}[Bounding quantile estimation error by dynamic regret]
    \label{lem:regret_distance}
    Fix any interval $I = [\tau, \tau + k - 1] \subseteq [T]$, and let $\Delta_I = \min\{\Delta_t : t \in I\}$. Under Assumption~\ref{ass:pdf_lb}, we have
    \begin{align*}
    \sum_{t \in I} \E\qty[\abs{\frac{\hat{s}_t - s^\star_t}{D}}]^q &\le 
    \frac{2q(q-1)}{bD} \qty(\frac{D}{2\Delta_I})^q \sum_{t \in I} \E\qty[\ell_{1-\alpha}(S_t, \hat{s}_t) - \ell_{1-\alpha}(S_t, s_t^\star)].
    \end{align*}
\end{lemma}

To prove Theorem~\ref{thm:cbce_coverage}, we follow a similar technique to \citet{zhang2018dynamic} and decompose the expected dynamic regret
\begin{align*}
    \E\qty[\sum_{t \in I} \ell_{1-\alpha}(S_t, \hat{s}_t) - \ell_{1-\alpha}(S_t, s_t^\star)] = \underbrace{\E\qty[\sum_{t \in I} \ell_{1-\alpha}(S_t, \hat{s}_t) - \inf_s \sum_{t \in I} \ell_{1-\alpha}(S_t, s)]}_{A} + \underbrace{\E\qty[\inf_s \sum_{t \in I} \ell_{1-\alpha}(S_t, s) - \ell_{1-\alpha}(S_t, s_t^\star)]}_{B}.
\end{align*}
We first observe that term $A$ is simply the expected interval regret on $I$. Since Proposition~\ref{prop:saregret} bounds the strongly adaptive regret with probability one, we can bound $A \le 15 D \sqrt{\abs{I} (\log T + 1)}$. Now, we analyze term $B$, and note that
\begin{align*}
\E\qty[\inf_s \sum_{t \in I} \ell_{1-\alpha}(S_t, s) - \ell_{1-\alpha}(S_t, s_t^\star)] &\le \inf_s \E\qty[\sum_{t \in I} \E_{\mathcal{F}_t}[\ell_{1-\alpha}(S_t, s) - \ell_{1-\alpha}(S_t, s_t^\star))]] 
\end{align*}
by Jensen's inequality and the tower property of conditional expectation. Now, for any $t \in I$ and $s \in [0, D]$,
\begin{align*}
&\E_{\mathcal{F}_t}[\ell_{1-\alpha}(S_t, s) - \ell_{1-\alpha}(S_t, s_t^\star)] \\
&= \E_{\mathcal{F}_t}[ (1 - \alpha) (S_t - s) \one[S_t > s] + \alpha (s - S_t) \one[S_t \le s] - (1 - \alpha)(S_t - s_t^\star)\one[S_t > s_t^\star] - \alpha(s_t^\star - S_t) \one[S_t \le s_t^\star] ] \\
&= \begin{cases}
\E_{\mathcal{F}_t}[(S_t - s)\one[S_t > s] - \alpha(S_t - s) - \alpha (s_t^\star - S_t)] & S_t \le s_t^\star \\
\E_{\mathcal{F}_t}[(s - S_t)\one[S_t \le s] - (1 - \alpha) (s - S_t) - (1 - \alpha) (s_t^\star - S_t)] & S_t > s_t^\star
\end{cases} \\
&= \begin{cases}
\E_{\mathcal{F}_t}[(S_t - s)\one[s < S_t \le s_t^\star] + \alpha(s - s_t^\star)] & S_t \le s_t^\star \\
\E_{\mathcal{F}_t}[(s - S_t)\one[s_t^\star \le S_t \le s] - (1 - \alpha) (s - s_t^\star)] & S_t > s_t^\star \\
\end{cases} \\
&\overset{(i)}{=} \E_{\mathcal{F}_t}[(S_t - s)\one[s \le S_t \le s_t^\star] + (s - S_t)\one[s_t^\star \le S_t \le s]] + \alpha (s - s_t^\star) \P_{\mathcal{F}_t}[S_t \le s_t^\star] - (1 - \alpha) (s - s_t^\star) \P_{\mathcal{F}_t}[S_t \le s_t^\star] \\
&\overset{(ii)}{=} \E_{\mathcal{F}_t}[(S_t - s)\one[s \le S_t \le s_t^\star] + (s - S_t)\one[s_t^\star \le S_t \le s]] \\
&=\abs{\int_{s}^{s_t^\star} (x - s) f_t(x) \dd x} \overset{(iii)}{\le} \frac{L(s_t^\star - s)^2}{2D}
\end{align*}
Above, (i) uses the fact that $s_t^\star$ is $\mathcal{F}_t$-measurable, (ii) uses the fact that $\P_{\mathcal{F}_t}[S_t \le s_t^\star] = 1 - \alpha$ by definition, and (iii) uses the density upper bound $f_t(x) \le \frac{L}{D}$ (Assumption~\ref{ass:pdf_ub}). Therefore, taking the infimum over $s\in\R$ and by definition of ${\rm Var}_I$~\eqref{eqn:interval_variation}, we have $B \le \frac{LD}{2} \mathrm{Var}_I$, and
\begin{align*}
    \E\qty[\sum_{t \in I} \ell_{1-\alpha}(S_t, \hat{s}_t) - \ell_{1-\alpha}(S_t, s_t^\star)] &\le 15D \sqrt{\abs{I} (\log T + 1)} + \frac{LD}{2} \mathrm{Var}_I.
\end{align*}
We combine this result with the power-mean inequality, Lemma~\ref{lem:regret_distance}, and the facts that $(q(q - 1))^{1/q} = \mathcal{O}(1)$ and $(x + y)^{1/q} \le x^{1/q} + y^{1/q}$ to prove the first part of Theorem~\ref{thm:cbce_coverage},
\begin{align}
\label{eq:cbce_quantile_distance_full}
\frac{1}{\abs{I}} \sum_{t \in I} \E\qty[\abs{\frac{\hat{s}_t - s_t^\star}{D}}] 
&\le \qty(\frac{1}{\abs{I}} \sum_{t \in I} \E\qty[\abs{\frac{\hat{s}_t - s_t^\star}{D}}]^q)^{1/q}
\le \mathcal{O}\qty(\frac{D}{b^{1/q} \Delta_I} \qty(\frac{\log T}{\abs{I}})^{1/(2q)}) + \mathcal{O}\qty(\frac{D}{\Delta_I} \qty( \frac{L \mathrm{Var}_I}{b \abs{I}})^{1/q}).
\end{align}
To prove the desired coverage bound, we note that
\begin{align*}
\frac{1}{\abs{I}} \sum_{t \in I} \abs{\P[Y_t \in \hat{C}_t(X_t)] - (1 - \alpha)}
&= \frac{1}{\abs{I}} \sum_{t \in I} \abs{\E\qty[ \E_{\mathcal{F}_t}[ \one[S_t \le \hat{s}_t] - \one[S_t \le s_t^\star]] ]}
\le \frac{1}{\abs{I}} \sum_{t \in I} \E\qty[ \abs{F_t(\hat{s}_t) - F_t(s_t^\star)} ],
\end{align*}
where the final inequality uses Jensen's inequality and the fact that $\E_{\mathcal{F}_t}[\one[S_t \le s]] = F_t(s)$ is the CDF of $S_t \mid \mathcal{F}_t$. The result follows by combining~\eqref{eq:cbce_quantile_distance_full} with Assumption~\ref{ass:pdf_ub}, which implies that the CDF $F_t$ is $L/D$-Lipschitz.
\end{proof-of-theorem}

\begin{proof-of-lemma}[\ref{lem:regret_distance}]
We first consider the following general situation, where $S$ is any continuous random variable bounded in $[0, D]$ with a density $f$. Define $s^\star = \q{1-\alpha}{S}$. Let $g$ be the density of the normalized random variable $X = \frac{2S - D}{D} \in [-1, 1]$, and let $x^\star = \q{1-\alpha}{X} = \frac{2s^\star - D}{D}$. As in \citet[Example 2.3]{steinwart2011pinball}, assume that there exist constants $b > 0, q \ge 2, \Delta > 0$ such that
\begin{align}
\label{eqn:pdf_s}
    g(x) \ge b \abs{x -x^\star}^{q-2} \iff f(s) \ge \frac{2b}{D} \abs{\frac{2(s - s^\star)}{D}}^{q - 2} 
\end{align}
for all $s \in [s^\star - \Delta, s^\star + \Delta]$. Let $\beta = \frac{b}{q-1}$ and  $\gamma = \beta (\frac{2\Delta}{D})^{q-1}$. By \citet[Theorem 2.7]{steinwart2011pinball},
\begin{align*}
    \abs{x - x^\star} &\le 2^{1-1/q} q^{1/q} \gamma^{-1/q} \qty(\E[\ell_{1-\alpha}(X, x) - \ell_{1-\alpha}(X, x^\star)] )^{1/q}  \\
    &= 2 \qty(\frac{q}{2 \gamma} \E[\ell_{1-\alpha}(X, x) - \ell_{1-\alpha}(X, x^\star)])^{1/q} \\
    &= 2 \qty(\frac{q(q-1)}{b} \qty(\frac{D}{2 \Delta})^{q-1} \E[\ell_{1-\alpha}(X, x) - \ell_{1-\alpha}(X, x^\star)] )^{1/q}
\end{align*}
Since $\abs{x - x^\star} = \frac{2}{D} \abs{s - s^\star}$, $\ell_{1-\alpha}(X, x) = \frac{2}{D}\ell_{1-\alpha}(S, s)$, and we can obtain
\begin{align*}
    \abs{\frac{s - s^\star}{D}}^q &\le \frac{2q(q-1)}{bD} \qty(\frac{D}{2 \Delta})^q \E[\ell_{1-\alpha}(S, s) - \ell_{1-\alpha}(S, s^\star)].
\end{align*}

Lemma~\ref{lem:regret_distance} now 
follows by fixing any $t \in I$, defining $\mathcal{G}_t = \sigma(\mathcal{F}_t, \hat{s}_t)$ and noticing that $S_t \mid \mathcal{G}_t \overset{\mathrm{dist}}{=} S_t \mid \mathcal{F}_t$ (as $\mathcal{G}_t$ only involves possibly an additional ``external'' randomness of the online prediction algorithm over $\cF_t$), so $s^\star_t = Q_{1-\alpha}(S_t \mid \mathcal{F}_t) = Q_{1-\alpha}(S_t \mid \mathcal{G}_t)$. Since $\hat{s}_t$ and $s_t^\star$ are $\mathcal{G}_t$-measurable, we can bound
\begin{align*}
    \E_{\mathcal{G}_t}\qty[\abs{\frac{\hat{s}_t - s_t^\star}{D}}^q] \le \frac{2q(q-1)}{bD} \qty(\frac{D}{2 \Delta_I})^q \E_{\mathcal{G}_t}[\ell_{1-\alpha}(S_t, \hat{s}_t) - \ell_{1-\alpha}(S_t, s_t^\star)].
\end{align*}
The result follows by taking unconditional expectations of both sides, observing that $\E[\abs{s_t - s_t^\star}]^q \le \E[\abs{s_t - s_t^\star}^q]$ by Jensen's inequality, and summing over all $t \in I$.
\end{proof-of-lemma}

\section{Additional Time Series Experiments}
\label{appendix:more_ts}
Here, we report the results of our time series experiments (as described in Section~\ref{sec:exp_time_series}) on M4 Weekly and NN5 Daily (the two smaller datasets) in Tables~\ref{tab:m4_weekly} and \ref{tab:nn5_daily}, respectively. The results on M4 Weekly (Table~\ref{tab:m4_weekly}) are quite similar to those on M4 Daily (Table~\ref{tab:m4_daily}). All methods do reasonably well on NN5. Considering that even split conformal attains strong worst-case local coverage error, the residuals likely have a near-exchangeable distribution on NN5 \citep[Theorems 2a, 3]{barber2022nexcp}.

\begin{table*}
    \centering
\begin{tabular}{l|cccc|cccc|cccc}
\toprule
 & \multicolumn{4}{c|}{LGBM (MAE = 0.19)} & \multicolumn{4}{c|}{ARIMA (MAE = 0.09)} & \multicolumn{4}{c}{Prophet (MAE = 0.41)} \\
Method & Coverage & Width & $\mathrm{LCE}_k$ & $\SAReg_k$ & Coverage & Width & $\mathrm{LCE}_k$ & $\SAReg_k$ & Coverage & Width & $\mathrm{LCE}_k$ & $\SAReg_k$ \\
\midrule
SCP & \color{red} 0.706 & 0.288 & 0.443 & 0.040 & \color{ForestGreen} 0.911 & 0.203 & 0.257 & 0.017 & \color{red} 0.555 & 0.453 & 0.571 & 0.058 \\
NExCP & \color{red} 0.764 & 0.268 & 0.411 & 0.014 & \color{ForestGreen} 0.904 & 0.185 & 0.254 & 0.009 & \color{red} 0.681 & 0.462 & 0.508 & 0.017 \\
FACI & \color{red} 0.822 & 0.254 & 0.282 & 0.008 & \color{ForestGreen} 0.899 & 0.169 & 0.205 & \bfseries 0.007 & \color{red} 0.776 & 0.458 & 0.338 & 0.007 \\
\methodBasic{} & \color{ForestGreen} 0.872 & 0.262 & 0.208 & 0.011 & \color{ForestGreen} 0.901 & 0.170 & \underline{0.191} & 0.010 & \color{ForestGreen} 0.871 & 0.475 & 0.209 & 0.009 \\
FACI-S & \color{ForestGreen} 0.863 & \bfseries 0.240 & \underline{0.207} & \underline{0.010} & \color{ForestGreen} 0.891 & \bfseries 0.152 & 0.197 & 0.010 & \color{ForestGreen} 0.856 & \bfseries 0.459 & \underline{0.197} & \bfseries 0.007 \\
\method{} & \color{ForestGreen} 0.872 & \underline{0.248} & \bfseries 0.170 & \bfseries 0.008 & \color{ForestGreen} 0.886 & \underline{0.155} & \bfseries 0.178 & \bfseries 0.007 & \color{ForestGreen} 0.868 & \underline{0.473} & \bfseries 0.158 & \bfseries 0.007 \\
\bottomrule
\end{tabular}
\caption{Results on M4 Weekly (359 time series) with target coverage $1 - \alpha = 0.9$ and interval size $k = 20$. Best results are {\bfseries bold}, while second best are \underline{underlined}, as long as the method's global coverage is in $(0.85, 0.95)$ (green). For all base predictors, \method{} achieves the best local coverage error, best strongly adaptive regret, and second-best width. The only methods which achieve global coverage in $(0.85, 0.95)$ for LGBM and Prophet are the ones that predict $\hat{s}_{t+1}$ directly, not as a quantile of $S_1, \ldots, S_t$.}
\label{tab:m4_weekly}
\end{table*}

\begin{table*}
    \centering
\begin{tabular}{l|cccc|cccc|cccc}
\toprule
 & \multicolumn{4}{c|}{LGBM (MAE = 0.08)} & \multicolumn{4}{c|}{ARIMA (MAE = 0.07)} & \multicolumn{4}{c}{Prophet (MAE = 0.11)} \\
Method & Coverage & Width & $\mathrm{LCE}_k$ & $\SAReg_k$ & Coverage & Width & $\mathrm{LCE}_k$ & $\SAReg_k$ & Coverage & Width & $\mathrm{LCE}_k$ & $\SAReg_k$ \\
\midrule
SCP & \color{ForestGreen} 0.932 & 0.205 & \bfseries 0.123 & 0.012 & \color{ForestGreen} 0.938 & 0.199 & \underline{0.126} & 0.012 & \color{ForestGreen} 0.912 & 0.221 & 0.154 & 0.010 \\
NExCP & \color{ForestGreen} 0.922 & 0.187 & 0.133 & 0.011 & \color{ForestGreen} 0.922 & 0.175 & 0.136 & 0.012 & \color{ForestGreen} 0.908 & 0.208 & 0.146 & 0.010 \\
FACI & \color{ForestGreen} 0.910 & \bfseries 0.179 & 0.130 & \bfseries 0.010 & \color{ForestGreen} 0.906 & \bfseries 0.162 & 0.132 & \bfseries 0.011 & \color{ForestGreen} 0.900 & \bfseries 0.200 & 0.131 & \bfseries 0.009 \\
\methodBasic{} & \color{ForestGreen} 0.904 & 0.190 & \bfseries 0.123 & 0.011 & \color{ForestGreen} 0.901 & 0.176 & 0.130 & 0.012 & \color{ForestGreen} 0.898 & 0.216 & \underline{0.128} & 0.011 \\
FACI-S & \color{ForestGreen} 0.909 & \bfseries 0.179 & 0.127 & \bfseries 0.010 & \color{ForestGreen} 0.910 & \underline{0.166} & \bfseries 0.123 & \bfseries 0.011 & \color{ForestGreen} 0.904 & \underline{0.203} & \bfseries 0.125 & \bfseries 0.009 \\
\method{} & \color{ForestGreen} 0.892 & \bfseries 0.179 & 0.125 & 0.012 & \color{ForestGreen} 0.895 & \underline{0.166} & 0.127 & 0.012 & \color{ForestGreen} 0.885 & 0.207 & \underline{0.128} & 0.011 \\
\bottomrule
\end{tabular}
\caption{Results on NN5 Daily (111 time series) with target coverage $1 - \alpha = 0.9$ and interval size $k = 20$. Best results are {\bfseries bold}, while second best are \underline{underlined}, as long as the method's global coverage is in $(0.85, 0.95)$ (green). All methods perform similarly well.}
\label{tab:nn5_daily}
\end{table*}

\section{Additional Experimental Details}
\label{appendix:exp_details}
We provide specific implementation details of all methods here. We use $Q_{1-\alpha}(\cdot)$ to denote the (empirical) $(1-\alpha)$-th quantile of a set of scalars, defined as
\begin{align}
\label{eq:def_Q}
Q_{1-\alpha}\paren{ \set{S_\tau}_{\tau=1}^t } \defeq \inf\set{s\in\R: \frac{1}{t}\sum_{\tau=1}^t \indic{S_\tau \le s} \ge 1-\alpha}.
\end{align}    

\begin{enumerate}[topsep=0pt, itemsep=0pt]
    \item SCP: Split conformal prediction \citep{vovk2005alrw} predicts $\hat{s}_{t+1} = Q_{1-\alpha}(\frac{1}{t}\sum_{\tau=1}^{t} \delta_{S_{t}})$.
    \item NExCP: Non-exchangeable conformal prediction extends SCP by using a {\em weighted} quantile function $\hat{s}_{t+1} = Q_{1-\alpha}(\frac{1}{t}\sum_{\tau=1}^{t} w_\tau \delta_{S_{t}})$. \citet{barber2022nexcp} suggest using geometrically decaying weights to adapt NExCP to situations with distribution shift, so we use $w_t = (1 - 3 \alpha / 4)^{1-t} w_1$.
    \item FACI: Fully Adaptive Conformal Inference has 4 hyperparameters: the individual expert learning rates $\gamma_1, \ldots, \gamma_N$; a target interval length $k$; and the meta-algorithm learning rate $\eta$; and a smoothing parameter $\sigma$. We set $k = 100$ and follow \citet{gibbs2022faci} to set $N = 8$, $\sigma = \frac{1}{2k}$, $\gamma = \{0.001, 0.002, 0.004, 0.008, 0.016, 0.032, 0.064, 0.128\}$, and
    \[\eta_t = \sqrt{\frac{\log(Nk) + 2}{\sum_{\tau=t-k}^{t-1}\E[\ell_\alpha(\beta_t, \alpha_t)^2]}},\]
    where the expectation is over $\alpha_t$. We also tried $k = 20$ for the time series experiments (to match our evaluation metrics), but the results were worse.
    \item \methodBasic{}: Scale-Free Online Gradient Descent. The only hyperparameter is the maximum radius $D$. For the time series experiments (Section~\ref{sec:exp_time_series}, Appendix~\ref{appendix:ensemble}), we set $D / \sqrt{3}$ for each horizon $h$ equal to the largest $h$-step residual observed on the calibration split of the training data. For the $m$-way image classification experiments (Section~\ref{sec:exp_image}, Appendix~\ref{appendix:more_cv}), we set $D = 1 + \lambda \sqrt{m - k_{reg}}$, where $\lambda$ and $k_{reg}$ are the width regularization parameters in~\eqref{eq:raps}.
    \item FACI-S: FACI applied to $S_t$ rather than $\alpha_t$. The hyperparameters are the same, except the losses used to compute $\eta_t$ are $\ell_{1-\alpha}(S_t, \hat{s}_t)$, and the learning rates are multiplied by $D$. We set $D$ in the same way as \methodBasic{}.
    \item \method{}: There are 2 hyperparameters: the maximum radius $D$ and the lifetime multiplier $g$ in~\eqref{eq:lifetime}. We set $D$ in the same way as \methodBasic{}. We set $g = 8$ for the time series experiments and $g = 32$ for the image classification experiments.
\end{enumerate}

\begin{table}[t]
\centering
\begin{tabular}{l|cccc|cccc|cccc}
\toprule
 & \multicolumn{4}{c|}{LGBM (MAE = 0.05)} & \multicolumn{4}{c|}{ARIMA (MAE = 0.14)} & \multicolumn{4}{c}{Prophet (MAE = 0.08)} \\
Method & Coverage & Width & $\mathrm{LCE}_k$ & $\SAReg_k$ & Coverage & Width & $\mathrm{LCE}_k$ & $\SAReg_k$ & Coverage & Width & $\mathrm{LCE}_k$ & $\SAReg_k$ \\
\midrule
EnbPI & \color{red} 0.803 & 0.088 & 0.331 & 0.017 & \color{ForestGreen} 0.916 & 0.220 & 0.186 & 0.032 & \color{red} 0.834 & 0.141 & 0.299 & 0.018 \\
EnbNEx & \color{ForestGreen} 0.864 & 0.097 & 0.218 & 0.010 & \color{ForestGreen} 0.907 & 0.198 & 0.193 & \underline{0.027} & \color{ForestGreen} 0.892 & 0.151 & 0.195 & 0.010 \\
EnbFACI & \color{ForestGreen} 0.856 & \bfseries 0.081 & 0.200 & \underline{0.007} & \color{ForestGreen} 0.900 & \bfseries 0.181 & \underline{0.161} & \bfseries 0.021 & \color{ForestGreen} 0.884 & \bfseries 0.130 & 0.164 & \underline{0.007} \\
Enb\methodBasic{} & \color{ForestGreen} 0.871 & 0.098 & \underline{0.173} & 0.008 & \color{ForestGreen} 0.906 & 0.201 & 0.165 & \underline{0.027} & \color{ForestGreen} 0.898 & 0.145 & \underline{0.144} & 0.008 \\
Enb\method{} & \color{ForestGreen} 0.884 & \underline{0.091} & \bfseries 0.134 & \bfseries 0.005 & \color{ForestGreen} 0.893 & \underline{0.192} & \bfseries 0.150 & 0.053 & \color{ForestGreen} 0.888 & \bfseries 0.130 & \bfseries 0.130 & \bfseries 0.005 \\
\bottomrule
\end{tabular}
\caption{Ensemble results on M4 Hourly with target coverage $1 - \alpha = 0.9$ and interval size $k = 20$. Best results are {\bfseries bold}, while second best are \underline{underlined}, as long as the method's global coverage is in $(0.85, 0.95)$ (green). Enb\method{} achieves the best local coverage error and strongly adaptive regret for all models, except ARIMA where its strongly adaptive regret is somewhat high.}
\label{tab:enb_m4_hourly}
\end{table}

\begin{table}[t]
    \centering
\begin{tabular}{l|cccc|cccc|cccc}
\toprule
 & \multicolumn{4}{c|}{LGBM (MAE = 0.11)} & \multicolumn{4}{c|}{ARIMA (MAE = 0.10)} & \multicolumn{4}{c}{Prophet (MAE = 0.18)} \\
Method & Coverage & Width & $\mathrm{LCE}_k$ & $\SAReg_k$ & Coverage & Width & $\mathrm{LCE}_k$ & $\SAReg_k$ & Coverage & Width & $\mathrm{LCE}_k$ & $\SAReg_k$ \\
\midrule
EnbPI & \color{red} 0.512 & 0.093 & 0.716 & 0.077 & \color{ForestGreen} 0.894 & 0.121 & 0.289 & 0.056 & \color{red} 0.791 & 0.288 & 0.374 & 0.062 \\
EnbNEx & \color{red} 0.749 & 0.142 & 0.520 & 0.020 & \color{ForestGreen} 0.894 & 0.116 & 0.293 & 0.035 & \color{red} 0.847 & 0.343 & 0.367 & 0.029 \\
EnbFACI & \color{red} 0.776 & 0.130 & 0.392 & 0.015 & \color{ForestGreen} 0.887 & \underline{0.100} & 0.248 & \bfseries 0.015 & \color{ForestGreen} 0.853 & \bfseries 0.232 & 0.230 & 0.025 \\
Enb\methodBasic{} & \color{red} 0.798 & 0.143 & 0.396 & 0.021 & \color{ForestGreen} 0.898 & 0.106 & \underline{0.241} & 0.047 & \color{ForestGreen} 0.900 & 0.292 & \underline{0.195} & \underline{0.023} \\
Enb\method{} & \color{ForestGreen} 0.875 & \bfseries 0.138 & \bfseries 0.203 & \bfseries 0.007 & \color{ForestGreen} 0.908 & \bfseries 0.096 & \bfseries 0.187 & \underline{0.034} & \color{ForestGreen} 0.917 & \underline{0.233} & \bfseries 0.139 & \bfseries 0.011 \\
\bottomrule
\end{tabular}
    \caption{Ensemble on M4 Daily with target coverage $1 - \alpha = 0.9$ and interval size $k = 20$. Best results are {\bfseries bold}, while second best are \underline{underlined}, as long as the method's global coverage is in $(0.85, 0.95)$ (green). Enb\method{} achieves the best or second best width, local coverage error, and strongly adaptive regret for all models. It is also the only method which achieves valid coverage for LGBM.}
    \label{tab:enb_m4_daily}
\end{table}

\begin{table}[t]
    \centering
\begin{tabular}{l|cccc|cccc|cccc}
\toprule
 & \multicolumn{4}{c|}{LGBM (MAE = 0.12)} & \multicolumn{4}{c|}{ARIMA (MAE = 0.07)} & \multicolumn{4}{c}{Prophet (MAE = 0.16)} \\
Method & Coverage & Width & $\mathrm{LCE}_k$ & $\SAReg_k$ & Coverage & Width & $\mathrm{LCE}_k$ & $\SAReg_k$ & Coverage & Width & $\mathrm{LCE}_k$ & $\SAReg_k$ \\
\midrule
EnbPI & \color{red} 0.540 & 0.118 & 0.596 & 0.079 & \color{ForestGreen} 0.893 & 0.174 & 0.266 & 0.021 & \color{red} 0.785 & 0.288 & 0.369 & 0.048 \\
EnbNEx & \color{red} 0.720 & 0.168 & 0.483 & 0.015 & \color{ForestGreen} 0.910 & 0.161 & 0.252 & 0.012 & \color{ForestGreen} 0.875 & 0.375 & 0.305 & 0.027 \\
EnbFACI & \color{red} 0.784 & 0.159 & 0.336 & 0.010 & \color{ForestGreen} 0.905 & \underline{0.142} & 0.193 & \underline{0.009} & \color{ForestGreen} 0.881 & \underline{0.235} & 0.201 & \underline{0.017} \\
Enb\methodBasic{} & \color{red} 0.795 & 0.172 & 0.329 & 0.016 & \color{ForestGreen} 0.908 & 0.152 & \underline{0.184} & 0.012 & \color{ForestGreen} 0.915 & 0.304 & \underline{0.164} & 0.023 \\
Enb\method{} & \color{ForestGreen} 0.874 & \bfseries 0.165 & \bfseries 0.173 & \bfseries 0.005 & \color{ForestGreen} 0.909 & \bfseries 0.131 & \bfseries 0.151 & \bfseries 0.008 & \color{ForestGreen} 0.907 & \bfseries 0.232 & \bfseries 0.133 & \bfseries 0.012 \\
\bottomrule
\end{tabular}
\caption{Ensemble results on M4 Weekly with target coverage $1 - \alpha = 0.9$ and interval size $k = 20$. Best results are {\bfseries bold}, while second best are \underline{underlined}, as long as the method's global coverage is in $(0.85, 0.95)$ (green). Enb\method{} achieves the best width, local coverage error, and strongly adaptive regret for all models. It is also the only method which achieves valid coverage for LGBM.}
\label{tab:enb_m4_weekly}
\end{table}

\begin{table}[t]
\centering
\begin{tabular}{l|cccc|cccc|cccc}
\toprule
 & \multicolumn{4}{c|}{LGBM (MAE = 0.09)} & \multicolumn{4}{c|}{ARIMA (MAE = 0.07)} & \multicolumn{4}{c}{Prophet (MAE = 0.07)} \\
Method & Coverage & Width & $\mathrm{LCE}_k$ & $\SAReg_k$ & Coverage & Width & $\mathrm{LCE}_k$ & $\SAReg_k$ & Coverage & Width & $\mathrm{LCE}_k$ & $\SAReg_k$ \\
\midrule
EnbPI & \color{ForestGreen} 0.859 & \bfseries 0.164 & 0.225 & 0.012 & \color{ForestGreen} 0.910 & \underline{0.163} & 0.159 & \underline{0.010} & \color{ForestGreen} 0.904 & \underline{0.163} & 0.160 & \bfseries 0.010 \\
EnbNEx & \color{ForestGreen} 0.882 & \underline{0.177} & 0.177 & \underline{0.010} & \color{ForestGreen} 0.920 & 0.174 & 0.138 & 0.012 & \color{ForestGreen} 0.923 & 0.178 & 0.126 & 0.012 \\
EnbFACI & \color{ForestGreen} 0.883 & \underline{0.177} & 0.166 & \bfseries 0.009 & \color{ForestGreen} 0.905 & \bfseries 0.156 & 0.131 & \bfseries 0.009 & \color{ForestGreen} 0.907 & \bfseries 0.156 & 0.123 & \bfseries 0.010 \\
Enb\methodBasic{} & \color{ForestGreen} 0.895 & 0.194 & \underline{0.132} & \underline{0.010} & \color{ForestGreen} 0.912 & 0.176 & \bfseries 0.120 & 0.012 & \color{ForestGreen} 0.912 & 0.178 & \underline{0.120} & 0.013 \\
Enb\method{} & \color{ForestGreen} 0.886 & 0.188 & \bfseries 0.131 & 0.012 & \color{ForestGreen} 0.906 & 0.168 & \underline{0.122} & 0.012 & \color{ForestGreen} 0.905 & 0.170 & \bfseries 0.118 & 0.013 \\
\bottomrule
\end{tabular}
\caption{Ensemble results on NN5 Daily with target coverage $1 - \alpha = 0.9$ and interval size $k = 20$. Best results are {\bfseries bold}, while second best are \underline{underlined}, as long as the method's global coverage is in $(0.85, 0.95)$ (green). Enb\method{} and Enb\methodBasic{} achieve the best and second-best worst-case local coverage error error at the cost of having slightly wider intervals.}
\label{tab:enb_nn5_daily}
\end{table}

\section{Time Series Experiments with Ensemble Models}
\label{appendix:ensemble}
In this section, we replicate the experiments of Section~\ref{sec:exp_time_series} using ensemble models trained with the method of EnbPI \citep{xu2021enbpi}. Specifically, we train base learner $\hat{f}^{(b)}$ on $(X_t, Y_T)_{t \in I_b}$, where $I_b$ is sampled randomly from $[T]$. Then, we obtain the residual $S_t^y = |y - \phi(\hat{f}^{(b)}(X_t) : I_b \not\ni t)|$ by aggregating all models not trained on $(X_t, Y_t)$. Finally, the residual of a new observation $(X_{T+1}, Y_{T+1})$ is $|Y_{T+1} - \phi(\hat{f}^{(b)}(X_{T+1}) : b \in [B])|$. We use $B = 5$ models in the ensemble.

EnbPI predicts the radius $\hat{s}_{t+1}$ as the $1 - \alpha$ empirical quantile of the previously observed residuals, as in split conformal prediction. However, these prediction sets can be obtained via an arbitrary function of the scores, including NExCP, FACI, \methodBasic{}, or \method{}. Besides EnbPI, we call these hybrid methods EnbNEx, EnbFACI, Enb\methodBasic{}, and Enb\method{} respectively. We use the same hyperparameters as described in Appendix~\ref{appendix:exp_details}.

This approach puts the other methods on even footing with EnbPI, because EnbPI removes the need for a train/calibration split and changes the underlying model from a single learner to a more accurate ensemble. We consider the contributions of EnbPI orthogonal to our own, and this section shows that their method can successfully be combined with ours.

The results mirror those of Section~\ref{sec:exp_time_series}. Enb\method{} generally obtains the best or second-best interval width, worst-case local coverage error, and strongly adaptive regret on all M4 datasets (Tables~\ref{tab:enb_m4_hourly}, \ref{tab:enb_m4_daily}, \ref{tab:enb_m4_weekly}). On NN5 (Table~\ref{tab:enb_nn5_daily}), all methods obtain similar strongly adaptive regret. However, Enb\methodBasic{} and Enb\method{} obtain the best and second-best worst-case local coverage error at the cost of having slightly wider intervals. Across the board, Enb\method{} has narrower intervals than Enb\methodBasic{}.

\begin{figure*}[t]
    \centering
    \includegraphics[width=0.9\textwidth,trim={0 0.7cm 0 0}]{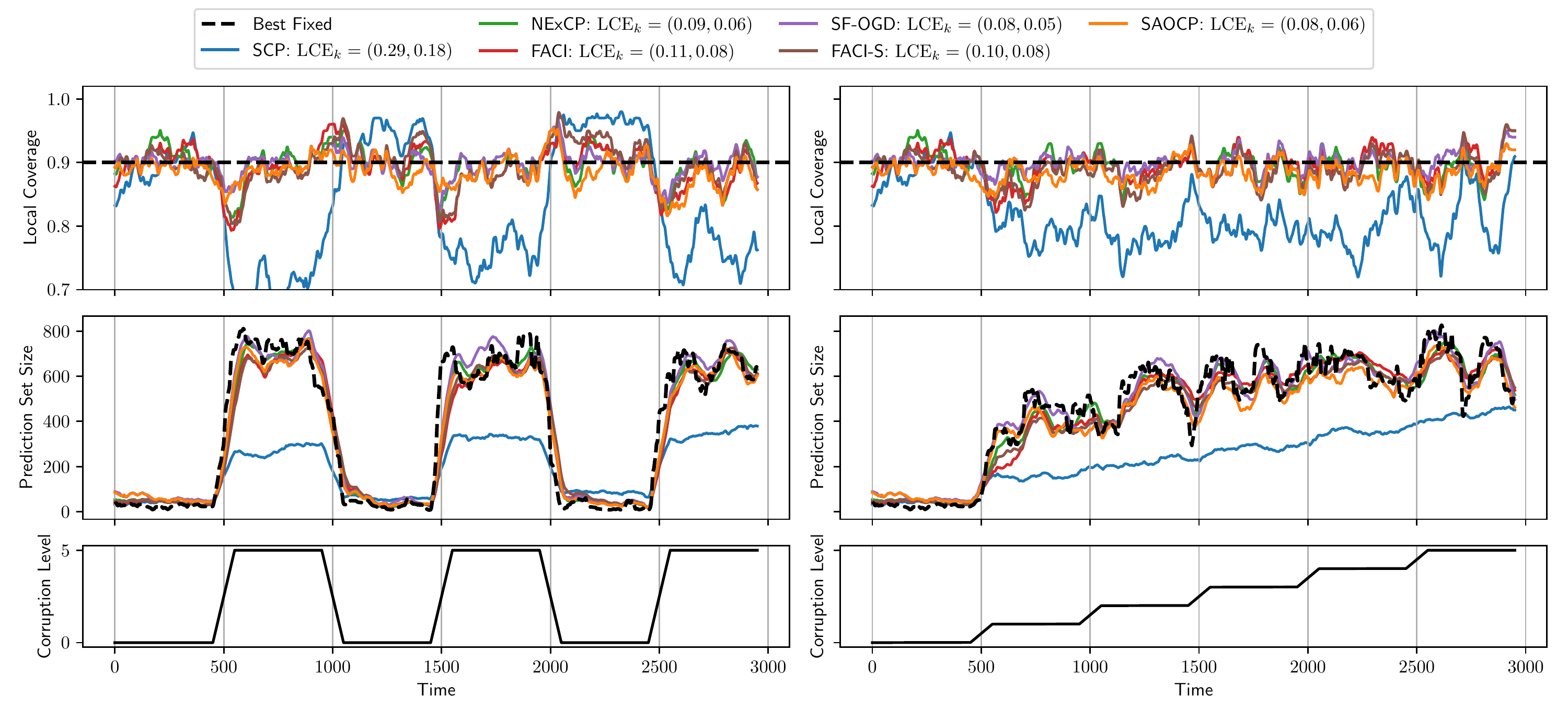}
    \caption{Local coverage (top row) and prediction set size (second row) achieved by various UQ methods when the distribution shifts between ImageNet and ImageNet-C every 500 steps. We plot moving averages with window size $k = 100$. Left: sudden shifts between corruption level 0 and 5. Right: gradual shift from level 0 to 5. \method{} and \methodBasic{}'s local coverage remain the closest to the target of 0.9, especially at the change points. While the two methods attain similar local coverage, \method{} returns smaller prediction sets than \methodBasic{}.
    }
    \label{fig:imagenet}
\end{figure*}

\section{Image Classification on TinyImageNet/TinyImageNet-C}
\label{appendix:more_cv}
We replicate the experiments of Section~\ref{sec:exp_image} using a ResNet-50 classifier on the TinyImageNet \citep{Le2015TinyIV} base dataset and its corrupted version TinyImageNet-C \citep{hendrycks2019benchmarking}. We train the model using SGD with learning rate 0.1 (annealed by a factor of 10 every 7 epochs), momentum 0.9, batch size 256, and early stopping if validation accuracy stops improving for 10 epochs. The model achieved a final test accuracy of 52.8\%. Figure~\ref{fig:tinyimagenet} shows the results. 

When performing uncertainty quantification, we use the conformal score~\eqref{eq:raps} with width regularization parameters $\lambda = 0.01$ and $k_{reg} = 20$. As in Section~\ref{sec:exp_image}, \method{} and \methodBasic{}'s local coverages remain the closest to the target of 0.9. The differences are most apparent when the distribution shift is sudden, suggesting that they are able to adapt to these distribution shifts more quickly than other methods. While all methods attain similar prediction set sizes, NExCP and FACI adapt more slowly to the best fixed prediction set size than \method{} and \methodBasic{}. \method{} also has better coverage than \methodBasic{}.

\end{document}